\documentclass{article}

\usepackage{kr}

\usepackage{times}
\usepackage{soul}
\usepackage{url}
\usepackage[hidelinks]{hyperref}
\usepackage[utf8]{inputenc}
\usepackage[small]{caption}
\usepackage{graphicx}
\usepackage{amsmath}
\usepackage{amsthm}
\usepackage{booktabs}
\usepackage{algorithm}
\usepackage{algorithmic}
\newtheorem{example}{Example}

\newtheorem{definition}{Definition}
\newtheorem{proposition}{Proposition}

\usepackage{tikz}

\usepackage{todonotes}

\newcommand{\inputs}{{\sf In}}
\newcommand{\outputs}{{\sf Out}}
\newcommand{\actions}{{\sf Acts}}
\newcommand{\intermediates}{{\sf Mid}}
\newcommand{\comestibles}{{\sf Coms}}

\newcommand{\nodes}{{\sf Nodes}}
\newcommand{\arcs}{{\sf Arcs}}
\newcommand{\graph}{{\sf Graph}}
\newcommand{\type}{{\sf Type}}

\newcommand{\compose}{{\sf Compose}}
\newcommand{\decompose}{{\sf Decompose}}

\newcommand{\cost}{{\sf Cost}}
\newcommand{\front}{{\sf Front}}

\title{A Graphical Formalism for Commonsense Reasoning with Recipes}

\author{%
Antonis Bikakis$^1$\and
Aissatou Diallo$^2$\and
Luke Dickens $^1$\and
Anthony Hunter$^2$
\and Rob Miller$^1$\\
\affiliations
$^1$Dept of Information Studies, University College London, London, UK\\
$^2$Dept of Computer Science, University College London, London, UK\\
\emails
\{a.bikakis, a.diallo, l.dickens, anthony.hunter, r.s.miller \}@ucl.ac.uk
}

\begin{document}

\maketitle

\begin{abstract}
Whilst cooking is a very important human activity, there has been little consideration given to how we can formalize recipes for use in a reasoning framework. We address this need by proposing a graphical formalization that captures the comestibles (ingredients, intermediate food items, and final products), and the actions on comestibles in the form of a labelled bipartite graph. We then propose formal definitions for comparing recipes, for composing recipes from subrecipes, and for deconstructing recipes into subrecipes. We also introduce and compare two formal definitions for substitution into recipes
which are required when there are missing ingredients, or some actions are not possible, or because there is a need to change the final product somehow. 
\end{abstract}

\section{Introduction}

Undeniably, cooking is an important activity for humans in order to live and to enjoy living.
But perhaps the commonplace nature of it has meant that little consideration has been given to how cooking involves reasoning with recipes. To address this shortcoming, we propose a high-level representation of recipes as labelled bipartite graphs where the first subset of nodes denotes the comestibles involved in the recipe (ingredients, intermediate food items, final products, i.e. dishes, and by-products) and the second subset of nodes denotes actions on those comestibles. The edges reflect the (possibly partial) sequence of steps taken in the recipe going from the ingredients to final products. The labelling on nodes is used to annotate the type of comestible or action for each node. 
Using our formal representation, we can consider how we can compose recipes from subrecipes, and decompose recipes into subrecipes. We can also consider comparing recipes to determine whether two recipes are equivalent, or whether one recipe is finer-grained than another.

We can also consider substitution in recipes (i.e. replacing comestibles and actions, or even replacing subrecipes, in recipes). 
We could require this because of lack of availability of some ingredients, dietary constraints, a need to reduce environmental impact, or an inability to carry out some actions.
For example, we could have a recipe for making bread that includes the
ingredient butter. If we do not have this ingredient, but we do have olive oil, then we can substitute olive oil for butter in the recipe.

Whilst our focus in this paper is on cooking, 
we see that substitution is an important issue across the gamut of human activities 
from every day home life (e.g.\ cooking, gardening, DIY, first-aid, etc), 
working life (e.g.\ farming, manufacturing, etc),
through to crisis management (e.g.\ dealing with the aftermath of earthquakes) \cite{Bikakis2021}.

There are other proposals for a graphical representation of recipes.
They are primarily used as a target language for natural language processing, 
or as a language for representing recipes in a corpus.
However, the meaning of the graphical notation is often unclear
as there can be multiple kinds of node and arc, 
and the difference between them is not formally defined. 
In contrast, we aim for a simple and clear syntax and semantics. 
Furthermore, there is no consideration in the literature of a formalism that supports the following kinds of reasoning:
(1) Operations to compose recipes from atomic recipes; 
(2) Operations to deconstruct complex recipes; 
(3) Formal methods for analyzing or comparing recipes;
And (4) Formal methods for substitution in recipes (including ingredients, actions, and subrecipes).

We proceed as follows:
Section \ref{section:types} introduces type hierarchies used for actions and comestibles; 
Section \ref{section:recipes} presents a representation of recipes as bipartite graphs;
Section \ref{section:acceptability} considers acceptability of recipes; 
Section \ref{section:comparison} presents definitions for comparing recipes;
Section \ref{section:composition} presents definitions for composition of recipes from subrecipes;
Section \ref{section:typesubstitution} presents substitution based on changing the type of nodes; 
Section \ref{section:structuralsubstitution} presents substitution based on changing the structure of the graph; 
Section \ref{section:literature} discusses related literature;
and Section \ref{section:discussion} discusses our proposal and future work.


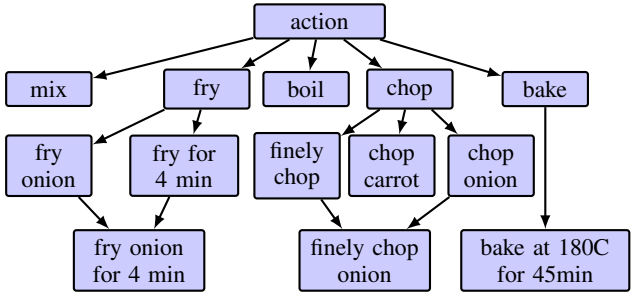
\begin{figure}
\begin{center}
\begin{tikzpicture}
[<-,>=latex,thick, scale=0.6,
onto/.style={draw,text centered, text width=9mm,
shape=rectangle, rounded corners=1pt,
fill=blue!20,font=\footnotesize}
]
\node[onto,text width=15mm] (a1) [] at (5,5) {action};
\node[onto] (b1) [] at (-1,3.5) {mix};
\node[onto] (b2) [] at (2.5,3.5) {fry};
\node[onto] (b2a) [] at (4.7,3.5) {boil};
\node[onto] (b3) [] at (7,3.5) {chop};
\node[onto] (b4) [] at (10,3.5) {bake};
\path	(b1) edge node[above] {} (a1);
\path	(b2) edge node[above] {} (a1);
\path	(b2a) edge node[above] {} (a1);
\path	(b3) edge node[above] {} (a1);
\path	(b4) edge node[above] {} (a1);
\node[onto] (c1) [] at (-1,1.8) {fry onion};
\node[onto] (c2) [text width=12mm] at (2,1.8) {fry for 4 min};
\node[onto] (c3) [] at (4.5,1.8) {finely chop};
\node[onto] (c4) [] at (6.6,1.8) {chop carrot};
\node[onto] (c5) [] at (8.8,1.8) {chop onion};
\path	(c1) edge node[above] {} (b2);
\path	(c2) edge node[above] {} (b2);
\path	(c3) edge node[above] {} (b3);
\path	(c4) edge node[above] {} (b3);
\path	(c5) edge node[above] {} (b3);
\node[onto,text width=15mm] (d1) [] at (1,-0.3) {fry onion for 4 min};
\node[onto,text width=15mm] (d2) [] at (6,-0.3) {finely chop onion};
\node[onto,text width=20mm] (d3) [] at (10,-0.3) {bake at 180C for 45min};
\path	(d1) edge node[above] {} (c1);
\path	(d1) edge node[above] {} (c2);
\path	(d2) edge node[above] {} (c3);
\path	(d2) edge node[above] {} (c5);
\path	(d3) edge node[above] {} (b4);
\end{tikzpicture}
\end{center}
\caption{\label{fig:hierarchy}An incomplete type hierarchy for action types. Each type is more specialized than its parent. The text in each node specifies the type, and may be written in different ways (e.g. ``bake at 180C for 45min" is equivalent to ``bake at 356F for 45min").}
\end{figure}


\begin{figure}
\begin{center}
\begin{tikzpicture}
[<-,>=latex,thick, scale=0.6,
onto/.style={draw,text centered, text width=9mm,
shape=rectangle, rounded corners=1pt,
fill=blue!20,font=\footnotesize}
]
\node[onto,text width=15mm] (a1) [] at (5,3) {comestible};
\node[onto] (b1) [] at (-0.5,1.5) {fish};
\node[onto] (b2) [] at (2.5,1.5) {pasta};
\node[onto] (b2a) [text width=12mm] at (5,1.5) {vegetable};
\node[onto] (b3) [] at (7.5,1.5) {fruit};
\node[onto] (b4) [] at (10.3,1.5) {diary};
\path	(b1) edge node[above] {} (a1);
\path	(b2) edge node[above] {} (a1);
\path	(b2a) edge node[above] {} (a1);
\path	(b3) edge node[above] {} (a1);
\path	(b4) edge node[above] {} (a1);
\node[onto] (c1) [text width=12mm] at (-0.2,0) {spaghetti};
\node[onto] (c2) [] at (2.5,0) {fusilli};
\node[onto] (c3) [] at (5,0) {onion};
\node[onto] (c4) [] at (7.5,0) {carrot};
\node[onto] (c5) [] at (10.3,0) {milk};
\path	(c1) edge node[above] {} (b2);
\path	(c2) edge node[above] {} (b2);
\path	(c3) edge node[above] {} (b2a);
\path	(c4) edge node[above] {} (b2a);
\path	(c5) edge node[above] {} (b4);
\node[onto] (d1) [text width=18mm] at (1,-1.5) {raw onion};
\node[onto] (d2) [text width=18mm] at (5,-1.5) {fried onion};
\node[onto] (d3) [text width=18mm] at (9.3,-1.5) {sliced onion};
\path	(d1) edge node[above] {} (c3);
\path	(d2) edge node[above] {} (c3);
\path	(d3) edge node[above] {} (c3);
\end{tikzpicture}
\end{center}
\caption{\label{fig:hierarchy2}An incomplete type hierarchy for comestible types. Each type is more specialized than its parent.}
\end{figure}
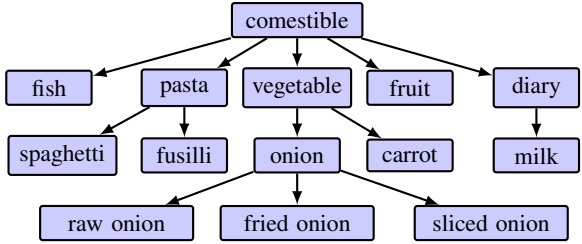

\section{Type hierarchies}
\label{section:types}

We focus on comestibles and actions as the two basic notions in a recipe.
We treat actions as atomic so that the details of duration, modifiers, equipment, etc are assumed to be part of the description of the action.
At a higher-level, we might consider actions such as ``mix" or ``bake", 
but at a lower-level, we might present more details. 
Similarly we treat comestibles as atomic, 
and so at higher-level we might have comestibles such as ``fish" or ``vegetable", 
then at a more detailed level we might have ``cherry tomato". 

We assume that each action and comestible is represented in a {\bf type hierarchy} which is an acyclic directed graph with a unique maximal node (root). 
Each node is a type
and each node in the hierarchy is a {\bf subtype} of its ancestors (e.g. Figures \ref{fig:hierarchy} and \ref{fig:hierarchy2}).
Let $t_1 \preceq t_2$ denote type $t_1$ is a subtype of, or equal to, type $t_2$,
(e.g. in Figure \ref{fig:hierarchy}, ``finely chop onion" is a subtype of ``chop"),
and let $t_1 \simeq t_2$ denote that either $t_1 \preceq t_2$ or $t_2 \preceq t_1$ holds (i.e they are on the same path). 
We also assume that comestibles are infinitely available. 
We do not consider quantities of them and so we assume there is always more than needed for a recipe.

\section{Recipes as graphs}
\label{section:recipes}

We represent the ``structure" of recipes as bipartite graphs based on a finite set of nodes $\cal C$, 
called comestible nodes, 
and a finite set of nodes $\cal A$, 
called action nodes, 
such that ${\cal C} \cap {\cal A} = \emptyset$.
For each action node, an incoming (respectively outgoing) node represents an input (respectively output) for the action.


\begin{definition}
A {\bf recipe graph} is a tuple $(C,A,E)$ where:
(1) $\emptyset \subset C \subseteq {\cal C}$ 
and $\emptyset \subset A \subseteq {\cal A}$; 
(2) $E$ is a set of arcs that is a subset of $(C \times A) \cup (A \times C)$; 
(3) $(C\cup A,E)$ is a connected acyclic graph; 
(4) for all $n_a \in A$, there are arcs $(n_c,n_a)$ and $(n_a,n'_c)$ in $E$;
and
(5) for all $n_c \in C$, if $(n_a,n_c),(n'_a,n_c) \in E$, then $n_a = n'_a$.
\end{definition}

A recipe graph cannot be an empty graph (condition 1). 
 For each action node, there is at least one incoming arc 
 and at least one outgoing arc (condition 4),
 and for each comestible node, there is at most one incoming arc (condition 5).
 Nodes are not necessarily unique to a recipe graph.
 So for recipe graphs $(C,A,E)$ and $(C',A',E')$,  
 it is possible for $(C\cup A) \cap (C'\cup A') \neq \emptyset$.
Conditon 2 ensures that recipe graphs are bipartite, 
but note that not all bipartite graphs are recipe graphs.

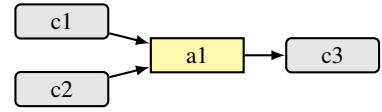
\begin{figure}
\begin{center}
\begin{tikzpicture}
[->,>=latex,thick, scale=0.6,
com/.style={draw,text centered, text width=10mm,
shape=rectangle, rounded corners=2pt,
fill=gray!20,font=\footnotesize},
act/.style={draw,text centered, 
shape=rectangle,text width=10mm,
fill=yellow!40,font=\footnotesize}
]
\node[com] (a1) [] at (0,1.5) {c1};
\node[com] (a1a) [] at (0,0) {c2};
\node[act] (a2) [] at (3,0.75)  {a1};
\node[com] (a3) [] at (6,0.75) {c3};
\path	(a1) edge node[above] {} (a2);
\path	(a1a) edge node[above] {} (a2);
\path	(a2) edge node[above] {} (a3);
\end{tikzpicture}
\end{center}
\caption{\label{fig:atomicrecipegraph}An atomic recipe graph where $c1$, $c2$, and $c3$ are comestible nodes and $a1$ is an action node.}
\end{figure}

\begin{definition}
An {\bf atomic recipe graph} is a recipe graph $(C,A,E)$ where $A$ contains exactly one action node.
\end{definition}

 We give an atomic recipe graph in Figure \ref{fig:atomicrecipegraph}
 and a non-atomic recipe graph in Figure \ref{fig:spaghetti:1}.

A recipe (as illustrated in Figure \ref{fig:spaghetti:1}) is a recipe graph with a typing function. 
This function assigns an action (respectively comestible) type to each action (respectively comestible) node,
s.t. that each type comes from the type hierarchy 
and no two comestible nodes get types from the same branch of the hierarchy 
(which ensures no comestible is the result of one or more actions on itself). 

\begin{definition}
\label{def:recipe}
A {\bf recipe} $R$ is a tuple $(C,A,E,F)$
where $(C,A,E)$ is a recipe graph
and $F$ is a {\bf typing function} that assigns an action (respectively comestible) type to each action (respectively comestible) node
s.t. for all $n,n'\in C$, if $F(n) \simeq F(n')$, 
then $n = n'$.
For a recipe $R = (C,A,E,F)$,
let $\graph(R) = (C,A,E)$, 
$\type(R) = F$, 
$\nodes(R) = C \cup A$,
and $\arcs(R) = E$. 
\end{definition}

A node with no incoming edges represents a comestible that has the role of being an {\bf ingredient} for the recipe (such as flour, eggs, sugar, or pesto sauce), and a node with no outgoing edges represents a comestible that has role of being a {\bf product} (the intended output such as a dish or something that can be used as an ingredient for another recipe) or a {\bf by-product} (i.e. a supplementary output) such as water from boiling potatoes which can be used as an ingredient for another recipe. A comestible node with incoming and outgoing edges represents a comestible that has the role of being an {\bf intermediate} item (i.e a comestible that is created and then used within the recipe).

Given a recipe graph $G$, we can use different typing functions. 
The choice of typing function dictates what the recipe is, and its detail.
For example, when considering an assignment to an action node using the hierarchy in Figure \ref{fig:hierarchy}, 
the choice of which branch to use (from the root) will fundamentally affect the recipe, whereas going down the branch will affect the level of detail of the recipe. Hence, for a node a1, typing functions $F_1$, $F_2$, and $F_3$, if $F_1$(a1) = fry, $F_2$(a1) = fry for 4 min, and $F_3$(a1) = bake, then $F_1$ and $F_2$ are assigning actions where one is a subtype of the other, whereas $F_3$ is giving a completely different action.

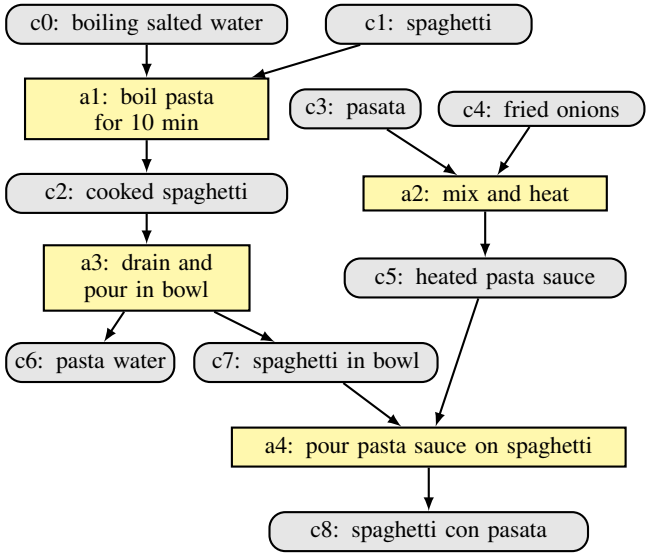
\begin{figure}
\begin{center}
\begin{tikzpicture}
[->,>=latex,thick, scale=0.75,
com/.style={draw,text centered, text width=40mm,
shape=rectangle, rounded corners=6pt,
fill=gray!20,font=\footnotesize},
act/.style={draw,text centered, 
shape=rectangle,text width=40mm,
fill=yellow!40,font=\footnotesize}
]
\node[com] (c0) [text width=35mm] at (0,8) {c0: boiling salted water};
\node[com] (c1) [text width=25mm] at (5,8) {c1: spaghetti};
\node[act] (a1) [text width=30mm] at (0,6.5) {a1: boil pasta for 10 min};
\node[com] (c2) [text width=35mm] at (0,5) {c2: cooked spaghetti};
\node[com] (c3) [text width=15mm] at (3.7,6.5) {c3: pasata};
\node[com] (c4) [text width=25mm] at (7,6.5) {c4: fried onions};
\node[act] (a2) [text width=30mm] at (6,5) {a2: mix  and  heat};
\node[com] (c5) [text width=35mm] at (6,3.5) {c5: heated  pasta sauce};
\node[act] (a3) [text width=25mm] at (0,3.5) {a3: drain and pour  in bowl};
\node[com] (c6b) [text width=20mm] at (-1,2) {c6: pasta water};
\node[com] (c6) [text width=30mm] at (3,2) {c7: spaghetti in bowl};
\node[act] (a4) [text width=50mm] at (5,0.5) {a4: pour pasta sauce on spaghetti};
\node[com] (c7) [text width=40mm] at (5,-1) {c8: spaghetti con pasata};
\path	(c0) edge node[above] {} (a1);
\path	(c1) edge node[above] {} (a1);
\path	(a1) edge node[above] {} (c2);
\path	(c3) edge node[above] {} (a2);
\path	(c4) edge node[above] {} (a2);
\path	(a2) edge node[above] {} (c5);
\path	(c2) edge node[above] {} (a3);
\path	(a3) edge node[above] {} (c6);
\path	(a3) edge node[above] {} (c6b);
\path	(c6) edge node[above] {} (a4);
\path	(c5) edge node[above] {} (a4);
\path	(a4) edge node[above] {} (c7);
\end{tikzpicture}
\end{center}
\caption{\label{fig:spaghetti:1}A recipe $R$ where for each node, its type is given after the colon.
Let $\graph(R) = G$.
So $G$ is a non-atomic recipe graph. 
Also, 
$\comestibles(G) = \{\mbox{\rm c0, c1, c2, c3, c4, c5, c6, c7, c8} \}$, 
$\actions(G) = \{\mbox{\rm a1, a2, a3, a4} \}$, 
$\inputs(G) = \{\mbox{\rm c0, c1, c3, c4} \}$, 
$\outputs(G) = \{ \mbox{\rm c6, c8} \}$, 
and $\intermediates(G) = \{\mbox{\rm c2, c5, c7} \}$.
}
\end{figure}

Some observations for a recipe graph $(C,A,E)$ are the following where ${\sf InDegree}(n)$ is the number of edges that impinge on $n$ and ${\sf OutDegree}(n)$ is the number of edges that exit $n$:
(1) for all $n \in A$, ${\sf InDegree}(n) > 0$ and ${\sf OutDegree}(n) > 0$; 
(2) for all $n \in C$, ${\sf InDegree}(n) \leq 1$;
(3) there is an $n \in C$ s.t. ${\sf InDegree}(n) = 0$; 
and (4) there is an $n \in C$ s.t. ${\sf OutDegree}(n) = 0$.

The following notation will be useful: 
For a recipe $R$, 
let $\actions(R)$ be the action nodes appearning in $R$; 
let $\inputs(R)$ be the comestible nodes that are the input nodes to $R$;
let $\outputs(R)$ be the comestible nodes that are the output nodes to $R$;
let $\intermediates(R)$ be the comestible nodes that are internal nodes in $R$;
and 
let $\comestibles(R)$ be the comestibles nodes appearing in $R$.
Therefore,
$\inputs(R)\cup\outputs(R)\cup\intermediates(R) = \comestibles(R)$; 
$\inputs(R)\cap\outputs(R) = \emptyset$; 
$\inputs(R)\cap\intermediates(R) = \emptyset$; 
and $\outputs(R)\cap\intermediates(R) = \emptyset$. 
As an example of using these functions, see Figure \ref{fig:spaghetti:1}.

Finally, for a recipe $R$, where $n,n'\in\nodes(R)$, 
we let $n \leq n'$ denote that there is a {\bf path} from $n$ to $n'$ in $\graph(R)$ or $n = n'$.
For example, in Figure \ref{fig:spaghetti:1}, 
some paths include $c_1 \leq c_7$, $c_3 \leq c_8$, and $c_5 \leq c_8$.

\section{Acceptability of recipes}
\label{section:acceptability}

The definition of recipes is liberal as it does not constrain what would be a sensible recipe.
As a starting point to addressing this issue, we can use existing recipes 
to determine what are acceptable inputs and outputs for each type of action. For example, for any action involving chopping or cutting, the input(s) have to be solid. As another example, if an input type is raw carrot, and the action type is chop, then the output type cannot be chopped onion.
We formalize this notion of acceptability in the rest of this section.

For this, we introduce the notion of an {\bf  acceptability tuple} $(t_1,t_2,t_3)$ where $t_1$ and $t_3$ are comestible types, and $t_2$ is an action type, 
so that for an action node $a$, with incoming (respectively outgoing) node $c$ (respectively $c'$), 
if the labelling function $F$ is such that $F(c) = t_1$, $F(a) = t_2$, and $F(c') = t_3$,
this labelling would be acceptable for these nodes. 
Examples of acceptability tuples that we might assume include the following. Note, the last two tuples involve types occurring in Figure \ref{fig:spaghetti:1}.
\[
\begin{array}{l}
\mbox{(bread, cut, slice of bread)}\\
\mbox{(slice of bread, put in toaster on medium, toast)}\\
\mbox{(raw carrot, chop, chopped carrot)}\\
\mbox{(pasata, mix and heat, heated pasta sauce)}\\
\mbox{(fried onions, mix and heat, heated pasta sauce)}\\
\end{array}
\]
The acceptability tuples may be inferred from other acceptability tuples and the type hierarchy. 
For example, from the tuple (carrot, chop, chopped carrot),
we might infer tuples where an action or a comestible is either a more general  
or more  specific type than the original such as the following.
\[
\begin{array}{l}
\mbox{(raw carrot, chop, chopped vegetable) }\\
\mbox{(raw carrot, chop, finely chopped carrot)} \\
\mbox{(raw carrot, finely chop, chopped carrot)} \\
\mbox{(raw carrot, cut in smaller pieces, chopped carrot)} \\
\mbox{(raw purple carrot, chop, chopped carrot)} \\
\end{array}
\]

In the following definition, we use a set of acceptability tuples to determine whether a recipe is acceptable.

\begin{definition}
A recipe $R$ is {\bf acceptable} w.r.t. the set of acceptability tuples $X$
iff for all $(c,a),(a,c') \in \arcs(R)$, $(F(c),F(a),F(c')) \in X$
where $\type(R) = F$.
\end{definition}


So satisfying a set of acceptability tuples is necessary but not sufficient for showing that a recipe makes sense. 
Of course, acceptability is subjective. Different cultures, cuisines, and tastes, are important in determining whether a recipe is acceptable (e.g. whether it is acceptable to use tomato ketchup as a sauce for pasta). Furthermore, an approach to acceptability that is prescriptive might inhibit innovation and creativity in cooking, and it might mean that useful substitutions are missed.
In future work, we will investigate inference of acceptability tuples from a commonsense model of cooking using non-monotonic reasoning.

\section{Comparison of recipes}
\label{section:comparison}

Given a set of recipes, natural questions to ask include 
whether one recipe is a subrecipe of another recipe, 
whether two recipes are equivalent in some sense, 
or whether one recipe is a finer-grained, 
or a more specific, recipe than another. 
We now consider these questions.
First, we introduce the following subsidiary functions 
that give the ingredients and final products of a recipe $R$: 
${\sf Inputs}(R) = \{ F(n) \mid n \in \inputs(R) \}$,
and ${\sf Outputs}(R) = \{ F(n) \mid n \in \outputs(R) \}$.

Recipes $R_1$ and $R_2$ are {\bf isomorphic} iff
there is a bijection $b: \nodes(R_1) \rightarrow \nodes(R_2)$,
s.t. $(n,n') \in \arcs(R_1)$ iff $(b(n),b(n')) \in \arcs(R_2)$.

One recipe is a subrecipe of another recipe means that the former is contained in the latter.

\begin{definition}
A recipe $R' = (C',A',E',F')$ is a {\bf subrecipe} of recipe $R = (C,A,E,F)$,
denoted $R' \sqsubseteq R$, 
iff 
$C' \subseteq C$,
$A' \subseteq A$,
$E' = E \cap (C' \times A') \cup (A' \times C')$,
and for all $n \in C'\cup A'$, $F'(n) = F(n)$. 
\end{definition}

Two recipes are equivalent if they are isomorphic and they have the same labels,
as illustrated in Figure \ref{fig:equivalent}. 

\begin{definition}
Recipes $R_1$ and $R_2$ are {\bf equivalent}, 
denoted $R_1 \equiv R_2$,
iff 
$R_1$ and $R_2$ are isomorphic
with bijection $b: \nodes(R_1) \rightarrow \nodes(R_2)$,
and for all $n \in \nodes(R_1)$, $F_1(n) = F_2(b(n))$, 
where $\type(R_1) = F_1$ and $\type(R_2) = F_2$.
\end{definition}

\begin{figure}
\begin{center}
\begin{tikzpicture}
[->,>=latex,thick, scale=0.75,
com/.style={draw,text centered, 
shape=rectangle, rounded corners=6pt,
fill=gray!20,font=\footnotesize},
act/.style={draw,text centered, 
shape=rectangle,text width=40mm,
fill=yellow!40,font=\footnotesize}
]
\node[com] (a1) [text width=20mm] at (0,1) {c1: raw onion};
\node[act] (a2) [text width=15mm] at (3.7,1) {a1: fry};
\node[com] (a3) [text width=25mm] at (7.8,1) {c2: fried onion};
\path	(a1) edge node[above] {} (a2);
\path	(a2) edge node[above] {} (a3);
\node[com] (b1) [text width=20mm] at (0,0) {c7: raw onion};
\node[act] (b2) [text width=15mm] at (3.7,0) {a8: fry};
\node[com] (b3) [text width=25mm] at (7.8,0) {c4: fried onion};
\path	(b1) edge node[above] {} (b2);
\path	(b2) edge node[above] {} (b3);
\end{tikzpicture}
\end{center}
\caption{\label{fig:equivalent}The top and bottom recipes are equivalent.}
\end{figure}
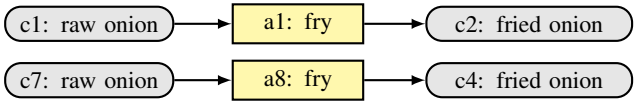

Two recipes are in-out aligned if
they input the same ingredients and output the same products  
but do not necessarily use the same methods to produce the outputs from the inputs (e.g. for making a white loaf, one uses the normal process of mixing, kneading, and baking, and the other uses a bread making machine).
For example, recipes in Figures \ref{fig:spaghetti:1} and \ref{fig:spaghetti:2} are in-out aligned. 

\begin{definition}
Recipe $R_1$ and $R_2$ are {\bf in-out aligned}, denoted $R_1 \equiv_{io} R_2$,
iff $\inputs(R_1)$ = $\inputs(R_2)$, 
$\outputs(R_1)$ = $\outputs(R_2)$,
and for each $c \in \inputs(R_1) \cup \outputs(R_1)$, $F_1(c) = F_2(c)$,
where $\type(R_1) = F_1$ and $\type(R_2) = F_2$.
\end{definition}


When comparing free-text written recipes for the same dish, it is common for some to provide more intermediate steps than another. We capture this in the definition below for one recipe being finer-grained than another (e.g. the recipe in Figure \ref{fig:spaghetti:1} is finer-grained than that in Figure \ref{fig:spaghetti:2}).

\begin{definition}
\label{def:finergrained}
For recipes $R_1$ and $R_2$, $R_1$ is {\bf finer-grained} than $R_2$, 
iff (1) $R_1 \equiv_{io} R_2$; 
and (2) there is a function $g: \nodes(R_1) \rightarrow \nodes(R_2)$
s.t. if $n\leq n'$ in $R_1$,
then $g(n) \leq g(n')$ in $R_2$.
\end{definition}

Condition 2 means that $g$ is an order-preserving map. 
By Definition \ref{def:finergrained}, for a recipe $R$, $R$ is finer grained than $R$.
Also, amongst an equivalence class of recipes as defined by the $\equiv_{io}$ relation, the atomic recipes are the least fine grained (i.e. for all recipes $R,R'$, if $R \equiv_{io} R'$, and $R$ is an atomic recipe, then $R'$ is finer-grained than $R$).

\begin{figure}
\begin{center}
\begin{tikzpicture}
[->,>=latex,thick, scale=0.75,
com/.style={draw,text centered, text width=40mm,
shape=rectangle, rounded corners=6pt,
fill=gray!20,font=\footnotesize},
act/.style={draw,text centered, 
shape=rectangle,text width=40mm,
fill=yellow!40,font=\footnotesize}
]
\node[com] (c0) [text width=20mm] at (0,2.8) {c0: boiling salted water};
\node[com] (c1) [text width=20mm] at (0,1.5) {c1: spaghetti};
\node[com] (c2) [text width=20mm] at (0,0.5) {c3: pasata};
\node[com] (c3) [text width=20mm] at (0,-0.5) {c4: fried onion};
\node[act] (a1) [text width=25mm] at (4,1) {a1: boil spaghetti for 10 minutes, and serve with warmed pasata and fried onions};
\node[com] (c4) [text width=20mm] at (8,1.7) {c8: spaghetti con pasata};
\node[com] (c6) [text width=20mm] at (8,0) {c6: pasta water};
\path	(c0) edge node[above] {} (a1);
\path	(c1) edge node[above] {} (a1);
\path	(c2) edge node[above] {} (a1);
\path	(c3) edge node[above] {} (a1);
\path	(a1) edge node[above] {} (c4);
\path	(a1) edge node[above] {} (c6);
\end{tikzpicture}
\end{center}
\caption{\label{fig:spaghetti:2}This recipe is in-out aligned to the recipe given in Figure \ref{fig:spaghetti:1} as they have the same types for the input and output nodes. But the recipe in Figure \ref{fig:spaghetti:1} is finer-grained than this recipe.}
\end{figure}
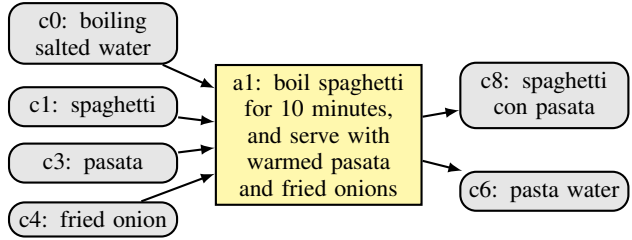

We can compare recipes 
 w.r.t. the type hierarchy.
As defined below, and illustrated in Figure \ref{fig:specificity}, 
this compares the specificity of the types assigned to two isomorphic graphs,
and so contrasts with the above notion of granularity 
that compares the structures of the graphs.
It provides a generalization of the notion of equivalence. 

\begin{definition}
For recipes $R_1$ and $R_2$, $R_1$ is {\bf more specific} than $R_2$, 
iff 
$R_1$ and $R_2$ are isomorphic 
with bijection $b$ from $\nodes(R_1)$ to $\nodes(R_2)$
s.t. for each node $n \in \nodes(R_1)$, 
$F(b(n)) \preceq F(n)$. 
\end{definition}

\begin{figure}
\begin{center}
\begin{tikzpicture}
[->,>=latex,thick, scale=0.75,
com/.style={draw,text centered, 
shape=rectangle, rounded corners=6pt,
fill=gray!20,font=\footnotesize},
act/.style={draw,text centered, 
shape=rectangle,text width=40mm,
fill=yellow!40,font=\footnotesize}
]
\node[com] (a1) [text width=20mm] at (0,1) {c1: raw onion};
\node[act] (a2) [text width=15mm] at (4,1) {a1: fry};
\node[com] (a3) [text width=23mm] at (8,1) {c2: fried onion};
\path	(a1) edge node[above] {} (a2);
\path	(a2) edge node[above] {} (a3);
\node[com] (b1) [text width=20mm] at (0,0) {c1: raw onion};
\node[act] (b2) [text width=25mm] at (3.9,0) {a2: fry for 4 min};
\node[com] (b3) [text width=23mm] at (8,0) {c2: fried onion};
\path	(b1) edge node[above] {} (b2);
\path	(b2) edge node[above] {} (b3);
\end{tikzpicture}
\end{center}
\caption{\label{fig:specificity}The bottom recipe is more specific than the top recipe.}
\end{figure}
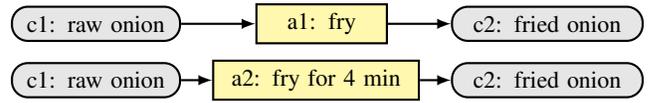

Whilst we have considered some ways of comparing recipes in this section,
there are numerous further ways that we can formalize comparison of recipes that could be useful for applications such as groupings of recipes according to the ingredients used, the actions used, the subrecipes that they contain, and the dishes that they produce.

\section{Composition of recipes}
\label{section:composition}

It is natural to think of a recipe being composed of subrecipes.
For instance, making a pasta dish with a sauce is composed of preparing the pasta 
and preparing the sauce in parallel, and then mixing them together.
To investigate composition, we start with the following definition.

\begin{definition}
For bipartite graphs $G_1 = (U_1,V_1,E_1)$ and $G_2 = (U_2,V_2,E_2)$, 
the {\bf bipartite union} of $G_1$ and $G_2$, 
is $G_1\uplus G_2$ = $(U_1 \cup U_2, V_1 \cup V_2, E_1 \cup E_2)$.
\end{definition}


\begin{figure}
\begin{center}
\begin{tikzpicture}
[->,>=latex,thick, scale=0.6,
com/.style={draw,text centered, text width=5mm,
shape=rectangle, rounded corners=6pt,
fill=gray!20,font=\footnotesize},
act/.style={draw,text centered, 
shape=rectangle,text width=5mm,
fill=yellow!40,font=\footnotesize}
]
\node[com] (c1) [] at (0,0) {c1};
\node[act] (a1) [] at (2,0) {a1};
\node[com] (c2) [] at (4,0) {c2};
\node[com] (c2a) [] at (8,0) {c2};
\node[act] (a2) [] at (10,0) {a2};
\node[com] (c1a) [] at (12,0) {c1};
\node[] (g1) [] at (3,-1) {$G_1$};
\node[] (g2) [] at (11,-1) {$G_2$};
\path	(c1) edge node[above] {} (a1);
\path	(a1) edge node[above] {} (c2);
\path	(c2a) edge node[above] {} (a2);
\path	(a2) edge node[above] {} (c1a);
\end{tikzpicture}
\begin{tikzpicture}
[->,>=latex,thick, scale=0.6,
com/.style={draw,text centered, text width=5mm,
shape=rectangle, rounded corners=6pt,
fill=gray!20,font=\footnotesize},
act/.style={draw,text centered, 
shape=rectangle,text width=5mm,
fill=yellow!40,font=\footnotesize}
]
\node[com] (c1) [] at (0,1) {c1};
\node[act] (a1) [] at (3,0) {a1};
\node[com] (c2) [] at (6,1) {c2};
\node[act] (a2) [] at (3,2) {a2};
\node[] (g1) [] at (3,-1) {$G_1 \uplus G_2$};
\path	(c1) edge node[above] {} (a1);
\path	(a1) edge node[above] {} (c2);
\path	(c2) edge node[above] {} (a2);
\path	(a2) edge node[above] {} (c1);
\end{tikzpicture}
\end{center}
\caption{\label{fig:union} Even though $G_1$ and $G_2$ are recipes graphs, the bipartite union of them, i.e. $G_1 \uplus G_2$, violates the definition for a recipe graph (because of the acyclic condition).} 
\end{figure}
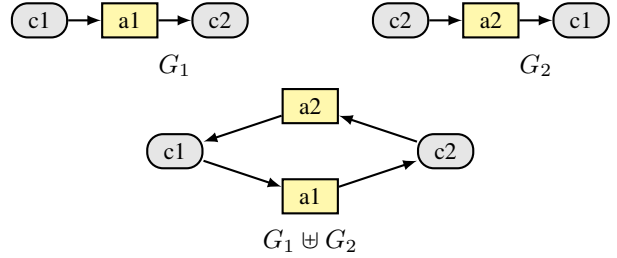

Some simple observations are 
(1) $G \uplus G = G$;
(2) $G_1 \uplus G_2 = G_2 \uplus G_1$;
(3) if $G_1$ is a subgraph of $G_2$, then $G_1 \uplus G_2 = G_2$;
and (4) if $G_1$ and $G_2$ are disjoint, 
then $G_1$ and $G_2$ are disjoint components in $G_1\uplus G_2$.
These seem to be desirable properties, but 
unfortunately the example in Figure \ref{fig:union} shows that bipartite union does not meet our needs for composing recipes since it can result in compositions that are bipartite graphs but not recipe graphs.

For this reason, we use the following definition for composition which we illustrate in Figure \ref{fig:composition:1} and Figure \ref{fig:composition:2}.
Essentially, two recipes $R_1$ and $R_2$ can be composed if the output nodes of $R_1$ overlap with the input nodes of $R_2$, and no other nodes are common between the two recipes. 

\begin{definition}
\label{def:composition}
If $R_1 = (C_1,A_1,E_1,F_1)$ and $R_2 = (C_2,A_2,E_2,F_2)$ are recipes
such that 
(1) $\outputs(R_1) \cap \inputs(R_2) \neq \emptyset$; 
(2) $\intermediates(R_1) \cap \intermediates(R_2) = \emptyset$; 
(3) $\actions(R_1) \cap \actions(R_2) = \emptyset$; 
(4) $\outputs(R_2) \cap \inputs(R_1) = \emptyset$; 
(5) for all $n \in \outputs(R_1) \cap \inputs(R_2), F_1(n) = F_2(n)$; 
and
(6) for all $n \in \comestibles(R_1)\setminus\outputs(R_1)$,
and for all $n' \in \comestibles(R_2)\setminus\inputs(R_2)$,
$F_1(n) \not\simeq F_2(n')$, 
then the {\bf composition} of $R_1$ and $R_2$, 
denoted by $R_1\oplus R_2 = (C,A,E,F)$, 
where
$C = C_1 \cup C_2$; 
$A = A_1 \cup A_2$; 
$E = E_1 \cup E_2$; 
and 
$F = F_1 \cup F_2$, 
otherwise $R_1\oplus R_2 = \bot$
(which denotes a failed composition). 
\end{definition}

We could generalize the definition by changing condition 5 to allow matching of subtypes rather than equality of types. 
We will investigate this more flexible way of combining recipes in future work.

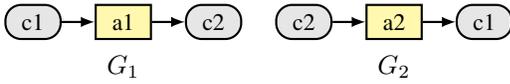
\begin{figure}
\begin{center}
\begin{tikzpicture}
[->,>=latex,thick, scale=0.6,
com/.style={draw,text centered, text width=5mm,
shape=rectangle, rounded corners=6pt,
fill=gray!20,font=\footnotesize},
act/.style={draw,text centered, 
shape=rectangle,text width=5mm,
fill=yellow!40,font=\footnotesize}
]
\node[com] (c1) [] at (0,0) {c1};
\node[act] (a1) [] at (2,0) {a1};
\node[com] (c2) [] at (4,0) {c2};
\node[com] (c2a) [] at (6,0) {c2};
\node[act] (a2) [] at (8,0) {a2};
\node[com] (c1a) [] at (10,0) {c1};
\node[] (g1) [] at (2,-1) {$G_1$};
\node[] (g2) [] at (8,-1) {$G_2$};
\path	(c1) edge node[above] {} (a1);
\path	(a1) edge node[above] {} (c2);
\path	(c2a) edge node[above] {} (a2);
\path	(a2) edge node[above] {} (c1a);
\end{tikzpicture}
\end{center}
\caption{\label{fig:composition:1}For recipes $R_1$ and $R_2$, where $\graph(R_1) = G_1$, 
and $\graph(R_2) = G_2$, $R_1 \oplus R_2 = \bot$ because of violation of Condition 4 in Definition \ref{def:composition}.} 
\end{figure}

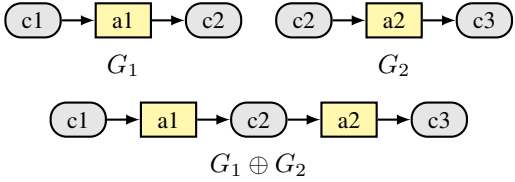
\begin{figure}
\begin{center}
\begin{tikzpicture}
[->,>=latex,thick, scale=0.6,
com/.style={draw,text centered, text width=5mm,
shape=rectangle, rounded corners=6pt,
fill=gray!20,font=\footnotesize},
act/.style={draw,text centered, 
shape=rectangle,text width=5mm,
fill=yellow!40,font=\footnotesize}
]
\node[com] (c1) [] at (0,0) {c1};
\node[act] (a1) [] at (2,0) {a1};
\node[com] (c2) [] at (4,0) {c2};
\node[com] (c2a) [] at (6,0) {c2};
\node[act] (a2) [] at (8,0) {a2};
\node[com] (c1a) [] at (10,0) {c3};
\node[] (g1) [] at (2,-1) {$G_1$};
\node[] (g2) [] at (8,-1) {$G_2$};
\path	(c1) edge node[above] {} (a1);
\path	(a1) edge node[above] {} (c2);
\path	(c2a) edge node[above] {} (a2);
\path	(a2) edge node[above] {} (c1a);
\end{tikzpicture}
\end{center}
\begin{center}
\begin{tikzpicture}
[->,>=latex,thick, scale=0.6,
com/.style={draw,text centered, text width=5mm,
shape=rectangle, rounded corners=6pt,
fill=gray!20,font=\footnotesize},
act/.style={draw,text centered, 
shape=rectangle,text width=5mm,
fill=yellow!40,font=\footnotesize}
]
\node[com] (c1) [] at (0,0) {c1};
\node[act] (a1) [] at (2,0) {a1};
\node[com] (c2) [] at (4,0) {c2};
\node[act] (a2) [] at (6,0) {a2};
\node[com] (c3) [] at (8,0) {c3};
\node[] (g1) [] at (4,-1) {$G_1 \oplus G_2$};
\path	(c1) edge node[above] {} (a1);
\path	(a1) edge node[above] {} (c2);
\path	(c2) edge node[above] {} (a2);
\path	(a2) edge node[above] {} (c3);
\end{tikzpicture}
\end{center}
\caption{\label{fig:composition:2}For recipes $R_1$ and $R_2$, where $\graph(R_1) = G_1$, 
and $\graph(R_2) = G_2$,  $R_1 \oplus R_2 \neq \bot$ 
when $F_1(C_1) = F_2(C_3)$} 
\end{figure}

Importantly, if the conditions of composition in Definition \ref{def:composition} are satisfied, 
then composition is a recipe.

\begin{proposition}
For recipe graphs $R_1$ and $R_2$,
if $R_1 \oplus R_2 \neq \bot$,
then $R_1 \oplus R_2$ is a recipe graph.
\end{proposition}

\begin{proof}
Assume
and $R_1$ and $R_2$ are recipes
and $R_1\oplus R_2 \neq \bot$.
So $R_1$ and $R_2$ satisfy conditions 1 to 5 in Definition \ref{def:composition}.
Since $C = C_1 \cup C_2$, $A = A_1 \cup A_2$, and $E = E_1 \cup E_2$,
$(C,A,E)$ is a bipartite graph.
Since $\outputs(R_1) \cap \inputs(R_2) \neq \emptyset$, 
the graph is connected.
Since $\intermediates(R_1) \cap \intermediates(R_2) = \emptyset$, 
$\actions(R_1) \cap \actions(R_2) = \emptyset$, 
and $\outputs(R_2) \cap \inputs(R_1) = \emptyset$,
the graph is acyclic.
Since for all $n \in \outputs(R_1) \cap \inputs(R_2), F_1(n) = F_2(n)$, 
$F = F_1 \cup F_2$ is well-formed.
So $R_1 \oplus R_2$ is a recipe. 
\end{proof}

As a consequence of the definition of composition, we get the following equivalences for the nodes.

\begin{proposition}
For recipes $R_1$ and $R_2$, 
if $R_1\oplus R_2 \neq\bot$, 
the following hold:
(1) $\inputs(R_1\oplus R_2) = (\inputs(R_1) \cup \inputs(R_2))\setminus (\outputs(R_1) \cap \inputs(R_2))$; 
(2) $\intermediates(R_1\oplus R_2) = (\outputs(R_1) \cap \inputs(R_2)) \cup \intermediates(R_1) \cup \intermediates(R_2)$; 
and
(3) $\outputs(R_1\oplus R_2) = (\outputs(R_1) \cup \outputs(R_2))\setminus (\outputs(R_1) \cap \inputs(R_2))$.
\end{proposition}

\begin{figure}
\begin{center}
\begin{tikzpicture}
[->,>=latex,thick, scale=0.6,
com/.style={draw,text centered, text width=25mm,
shape=rectangle, rounded corners=2pt,
fill=gray!20,font=\footnotesize},
act/.style={draw,text centered, 
shape=rectangle,text width=10mm,
fill=yellow!40,font=\footnotesize}
]
\node (r1) at (0,0) {R1};
\node[com] (c1) [text width=15mm] at (2,0) {c1: tomato};
\node[act] (a1) [text width=15mm] at (6,0)  {a1: chop};
\node[com] (c2) [text width=27mm] at (11,0) {c2: chopped tomato};
\path	(c1) edge node[above] {} (a1);
\path	(a1) edge node[above] {} (c2);
\end{tikzpicture}

\vspace{3mm}

\begin{tikzpicture}
[->,>=latex,thick, scale=0.6,
com/.style={draw,text centered, text width=25mm,
shape=rectangle, rounded corners=2pt,
fill=gray!20,font=\footnotesize},
act/.style={draw,text centered, 
shape=rectangle,text width=10mm,
fill=yellow!40,font=\footnotesize}
]
\node (r2) at (0,0) {R2};
\node[com] (c3) [text width=15mm] at (2,0) {c3: lettuce};
\node[act] (a2) [text width=15mm] at (6,0)  {a2: chop};
\node[com] (c4) [text width=27mm] at (11,0) {c4: chopped lettuce};
\path	(c3) edge node[above] {} (a2);
\path	(a2) edge node[above] {} (c4);
\end{tikzpicture}

\vspace{3mm}

\begin{tikzpicture}
[->,>=latex,thick, scale=0.6,
com/.style={draw,text centered, text width=25mm,
shape=rectangle, rounded corners=2pt,
fill=gray!20,font=\footnotesize},
act/.style={draw,text centered, 
shape=rectangle,text width=10mm,
fill=yellow!40,font=\footnotesize}
]
\node (r3) at (-3,0.75) {R3};
\node[com] (c2) [text width=28mm] at (0,1.5) {c2: chopped tomato};
\node[com] (c4) [text width=28mm] at (0,0) {c4: chopped lettuce};
\node[act] (a3) [text width=15mm] at (5,0.75)  {a3: mix};
\node[com] (c5) [text width=15mm] at (9,0.75) {c5: salad};
\path	(c2) edge node[above] {} (a3);
\path	(c4) edge node[above] {} (a3);
\path	(a3) edge node[above] {} (c5);
\end{tikzpicture}

\end{center}
\caption{\label{fig:nonassociative}. For recipes $R1$, $R2$, and $R3$, $R1 \oplus (R2 \oplus R3)$ is a valid recipe but $(R1 \oplus R2) = \bot$ and so $(R1 \oplus R2) \oplus R3$ is not valid. Hence, $\oplus$ is not associative.}
\end{figure}
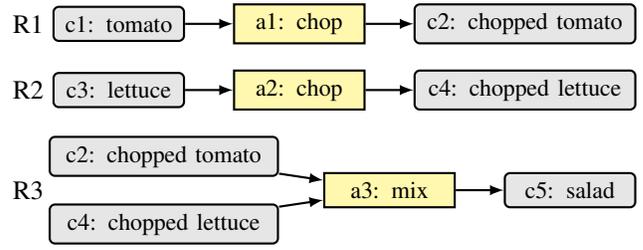

Unlike bipartite union, $\oplus$ is not commutative, nor associative (as illustrated in Figure \ref{fig:nonassociative}).
Furthermore, if $R_1 \oplus R_2 \neq \bot$, 
then $R_2 \oplus R_1 = \bot$.
Also if $R_1$ and $R_2$ are disjoint (i.e. $\nodes(R_1) \cap \nodes(R_2) = \emptyset$), 
then $R_1 \oplus R_2 = \bot$ 
and $R_2 \oplus R_1 = \bot$.

Any recipe can be composed from one or more atomic recipes (i.e. recipes with an atomic recipe graph).
To show this, we use the ${\sf Compose}$ function which for a set of 
recipes gives the closure under the $\oplus$ operator.
We define this as follows:
For any set of recipes $\{R_1,\ldots,R_n\}$, 
and for any recipes $R, R' \in \compose(\{R_1,\ldots,R_n\})$, 
$R\oplus R' \in \compose(\{R_1,\ldots,R_n\})$.

\begin{proposition}
For all recipes $R$, there is a set of atomic recipes $\{R_1,\ldots,R_n\}$
s.t. $R \in \compose(\{R_1,\ldots,R_n\})$.
\end{proposition}

\begin{proof}
Since $R$ is a recipe, it is a labelled connected acyclic bipartite graph.
For each action node, the inputs and outputs correspond to an atomic recipe.
So for each action, there is a corresponding atomic recipe $R_i\in\{R_1,\ldots,R_n\}$.
So $R \in \compose(\{R_1,\ldots,R_n\})$.
\end{proof}

Only a finite number of recipe graphs can be composed from a finite set of atomic recipe graphs
since infinite sequences cannot be formed.

\begin{example}
Consider the first two atomic recipes below. 
We can compose them so that we start with fresh peas and finish with thawed peas. 
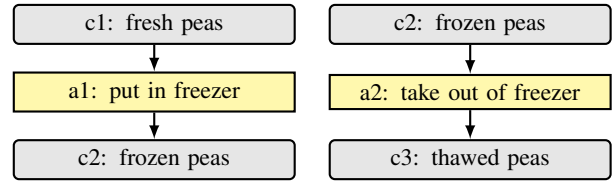
\begin{center}
\begin{tikzpicture}
[->,>=latex,thick, scale=0.6,
com/.style={draw,text centered, text width=25mm,
shape=rectangle, rounded corners=2pt,
fill=gray!20,font=\footnotesize},
act/.style={draw,text centered, 
shape=rectangle,text width=10mm,
fill=yellow!40,font=\footnotesize}
]
\node[com] (c1) [text width=35mm] at (0,3) {\rm c1: fresh peas};
\node[act] (a1) [text width=35mm] at (0,1.5)  {\rm a1: put in freezer};
\node[com] (c2) [text width=35mm] at (0,0) {\rm c2: frozen peas};
\path	(c1) edge node[above] {} (a1);
\path	(a1) edge node[above] {} (c2);
\node[com] (c1r) [text width=35mm] at (7,3) {\rm c2: frozen peas};
\node[act] (a1r) [text width=35mm] at (7,1.5)  {\rm a2: take out of freezer};
\node[com] (c2r) [text width=35mm] at (7,0) {\rm c3: thawed peas};
\path	(c1r) edge node[above] {} (a1r);
\path	(a1r) edge node[above] {} (c2r);
\end{tikzpicture}
\end{center}
We could extend the example with the following atomic recipe on the left to obtain a recipe for refrozen peas. But we would not be able to extend it further with the atomic recipe on the right because it would violate the conditions for composition (because two comestible nodes would have the same type). 
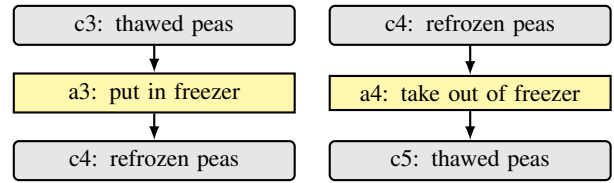
\begin{center}
\begin{tikzpicture}
[->,>=latex,thick, scale=0.6,
com/.style={draw,text centered, text width=25mm,
shape=rectangle, rounded corners=2pt,
fill=gray!20,font=\footnotesize},
act/.style={draw,text centered, 
shape=rectangle,text width=10mm,
fill=yellow!40,font=\footnotesize}
]
\node[com] (c1) [text width=35mm] at (0,3) {\rm c3: thawed peas};
\node[act] (a1) [text width=35mm] at (0,1.5)  {\rm a3: put in freezer};
\node[com] (c2) [text width=35mm] at (0,0) {\rm c4: refrozen peas};
\path	(c1) edge node[above] {} (a1);
\path	(a1) edge node[above] {} (c2);
\node[com] (c1r) [text width=35mm] at (7,3) {\rm c4: refrozen peas};
\node[act] (a1r) [text width=35mm] at (7,1.5)  {\rm a4: take out of freezer};
\node[com] (c2r) [text width=35mm] at (7,0) {\rm c5: thawed peas};
\path	(c1r) edge node[above] {} (a1r);
\path	(a1r) edge node[above] {} (c2r);
\end{tikzpicture}
\end{center}
Note, for this example, we assume that fresh peas, frozen peas, thawed peas, and refrozen peas are in different branches of the type hierarchy for comestibles. 
\end{example}

\begin{proposition}
For any finite set of atomic recipes $\{R_1,\ldots,R_n\}$, 
the set $\compose(\{R_1,\ldots,R_n\})$ is finite.
\end{proposition}

\begin{proof}
Assume $\{R_1,\ldots,R_n\}$ is finite. 
The constraints on composition ensure that 
if $R'$ is a subrecipe of $R$,
then $R \oplus R'$ is not in $\compose(\{R_1,\ldots,R_n\})$.
So for each $R$ in $\compose(\{R_1,\ldots,R_n\})$,
each $R_i \in \{R_1,\ldots,R_n\}$
is used at most once.
Hence, each $R \in \compose(\{R_1,\ldots,R_n\})$
is composed from a finite number of subrecipes,
where none is used more than once.
Hence, there is a finite number of recipes in $\compose(\{R_1,\ldots,R_n\})$.
\end{proof}


So far, we have considered how to compose subrecipes into recipes.
Now, we consider the inverse process for turning recipes into atomic subrecipes.
For a recipe $R$, the {\bf decompose function}, denoted $\decompose(R)$, returns the set $\{ R' \mid R' \sqsubseteq R \mbox{ and } R' \mbox{ is atomic} \}$.

\begin{proposition}
For any $R$, if $R$ is a recipe, then $R \in \compose(\decompose(R))$.
\end{proposition}

\begin{proof}
Any recipe graph can be (de)composed in the form of a binary tree where each leaf is an atomic recipe graph, and each non-leaf is the composition of its children, and the root is the recipe. 
Let $\{R_1,\ldots,R_n\}$ be the leaves, and $R$ be the root. 
So $\decompose(R) = \{R_1,\ldots,R_n\}$
and $R \in \compose(\{R_1,\ldots,R_n\})$. 
\end{proof}

So we have provided functions for composing and decomposing recipes.
In this, we see that atomic recipes are the basic building blocks for recipes.

\section{Type substitution in recipes}
\label{section:typesubstitution}

We now consider how we can substitute individual comestibles and actions. 
In other words, how we can update the typing function.

A {\bf binding} is a tuple of the form $(n,t)$ where $n \in {\cal C} \cup {\cal A}$ is a node and $t$ is a type from a type hierarchy.
A {\bf substitution set} $T$ is a set of bindings s.t. for all $(n,t),(n',t') \in T$, if $n \neq n'$, 
then $t = t'$ (i.e. each binding refers to a different node).

\begin{definition}
\label{def:typesubstitution}
For a recipe $R$, where $\type(R) = F$, 
and a substitution set $T$, 
the {\bf substitution} in $F$ by $T$, 
denoted $F\otimes T$, 
is defined as follows for each $n \in \nodes(R)$.
\[
F \otimes T(n) =
\left\{
\begin{array}{ll}
t & \mbox{ if } (n,t) \in T\\
F(n) & \mbox{ otherwise } 
\end{array}
\right.
\] 
For convenience, we let $R\otimes T$ denote the recipe $R'$ where 
$\graph(R') = \graph(R)$ 
and $\type(R') = F\otimes T$. 
\end{definition}

\begin{table*}
\begin{center}
\begin{tabular}{|l|lllll|}
\hline
& c1 & a1 & c2 & a2 & c3\\
\hline
$F$ & raw carrot & chop & chopped carrot & boil & soup\\
\hline
$F' = F \otimes \{ (c1,\mbox{raw onion}) \}$ & raw onion & chop & chopped carrot & boil & soup\\
\hline
$F'' = F' \otimes \{ (c2,\mbox{chopped onion}) \}$ & raw onion & chop & chopped onion & boil & soup\\
\hline
\end{tabular}
\end{center}
\caption{\label{tab:substitution:1} 
Consider a recipe graph $G$ with 
nodes $\{\rm c1,a1,c2,a2,c3\}$
and edges $\rm (c1,a1),(a1,c2),(c2,a2),(a2,c3)$.
The typing function $F$ is given in the first row of the table.
The second row is for the updated typing function $F'$ 
which results from a primary substitution set $\{$($c1$, raw onion)$\}$.
The third row is for the updated typing function $F''$ 
which results from a second substitution set $\{$($c2$,chopped onion)$\}$.
}
\end{table*}

For an illustration of type substitution, see Table \ref{tab:substitution:1}, where the first row gives the assignment for the typing function $F$, and the second row gives the updated assignment after the substitution of ``raw onion" for node c1.

Type substitutions allow any recipe to be turned into any other isomorphic recipe.
But note that type substitution may cause the recipe to violate the acceptability tupes.
We discuss how we deal with this issue after the following result.

\begin{proposition}
For any recipes $R$ and $R'$, 
if $\graph(R)$ and $\graph(R')$ are isomorphic,
then there is a substitution set $T$
s.t. $F \otimes T = F'$
where $\type(R) = F$ and $\type(R') = F'$
\end{proposition}

\begin{proof}
Assume $R$ and $R'$ are isomorphic.
So they have the same structure, 
though possibly the names of the nodes and their labelling are different.
Let $b: \nodes(R) \rightarrow \nodes(R')$ be a bijection, 
and let $T = \{(n,t) \mid F(n) \neq F'(b(n)) \mbox{ and } F'(b(n)) = t\}$.
So $F \otimes T = F'$.
\end{proof}


We now consider how the need for substitutions arises in practice. 
A {\bf primary substitution} is a substitution that has been undertaken because we lack some food item or
we are unable to do an action (perhaps because we lack required equipment or ability), whereas a {\bf secondary
substitution} is a substitution that has been carried out to deal with acceptability issues raised by the
primary substitution.
For example, suppose we lack fresh spaghetti for a recipe that lists it as an ingredient.
We could use dried spaghetti as a substitute for this missing ingredient.
This would be a primary substitution.
However, the cooking time of fresh spaghetti is 3 minutes whereas the cooking time of dried spaghetti is 11 minutes.
Assuming we have appropriate acceptability tuples, this would result in a violation of the acceptability tuples, 
and so we would need to substitute the action ``boil spaghetti for 3 minutes" to ``boil spaghetti for 11 minutes". This would be a secondary substitution which would be required for the recipe to regain acceptability. 

A {\bf primary substitution set}, denoted $P$, is a set of primary substitutions.
For instance, 
suppose we have a recipe $R$, and we are missing ingredients $F(n_1),\ldots,F(n_k)$ 
where $n_1,\ldots,n_k$ are nodes in $\inputs(R)$,  
then we would need alternative ingredients $t_1,\ldots,t_k$ 
to give the primary substitution set $\{ (n_1,t_1),\ldots,(n_k,t_k) \}$. 
As defined next, a secondary substitution set is needed to fix any acceptability problems created by the primary substitution set.

\begin{definition}
A substitution set $S$ is a {\bf secondary substitution set} for recipe $R$, 
primary substitution set $P$, 
and acceptability tuples $X$ 
iff $P \cup S$ is a substitution set,
and $R\otimes (P \cup S)$ is acceptable w.r.t. $X$. 
\end{definition}

An illustration of the use of primary and secondary substitution sets is given in Table \ref{tab:substitution:1}.
The next proposition follows directly from Definition \ref{def:typesubstitution}.

\begin{proposition}
For a recipe $R$, where 
$\type(R) = F$,
the following hold:
(Reflexivity) 
$F\otimes \{(n,F(n))\} = F$; 
(Associativity) 
if $n_1 \neq n_2$, then
$(F\otimes\{(n_1,t_1)\})\otimes\{(n_2,t_2)\}$ 
=  $(F\otimes\{(n_2,t_2)\})\otimes\{(n_1,t_1)\}$; 
(Reversibility) 
$(F\otimes\{(n,t)\})\otimes\{(n,F(n))\} = F$;
and (Empty) 
if $T = \{ (n_1,t_1),\ldots,(n_i,t_i)\}$, 
and $n_1,\ldots, n_i \not\in \nodes(R)$, 
then $F\otimes T = F$.
\end{proposition}


A desirable feature of updating a typing function is that unnecessary updates are not done.
So we seek the minimal (by subset) secondary substitution sets as defined next.

\begin{definition}
A {\bf substitution pair} for a recipe $R$ and acceptability set $X$ is a tuple $(P,S)$ 
where $P$ (respectively $S$) is a primary (respectively secondary) substitution set
and $R\otimes(P \cup S)$ is acceptable w.r.t. $X$
and there is no $S' \subset S$ such that $R\otimes(P \cup S')$ is acceptable w.r.t. $X$. 
\end{definition}

However, the above definition does not take into account the nature of individual substitutions. 
For instance, if we lack spaghetti for spaghetti bolognese, but we have tagliatelle and rice, either would be possible substitutes, but many would judge tagliatelle to be a much less drastic change to the recipe. 
To address this issue, we use a distance measure to compare comestibles and actions. So for comestible or action types $t_1$ and $t_2$, $d(t_1,t_2)$ denotes the distance between $t_1$ and $t_2$. The smaller the distance, the better one would substitute for another. So $d(t_1,t_2) = 0$ means $t_1$ and $t_2$ would be perfect substitutes for each other. Since it is a distance measure, it is always the case that $d(t,t) = 0$, and so any comestible or action is a perfect substitute for itself. 
It may also be appropriate in some cases to define $d$ with respect to a recipe so that, for example, aquafaba might reasonably replace egg in baking but not in an omelette.

A distance measure can be defined based on a word embedding such as the general purpose word embeddings Word2Vec \cite{Mikolov2013} or Glove \cite{Pennington2014}, or a specialized word embedding such as Food2Vec which is a pre-trained word embedding for ingredient substitution \cite{Pellegrini2021}. For comestibles $t_1$ and $t_2$, and a word embedding, the distance function $d(t_1,t_2)$ is the cosine similarity between $t_1$ and $t_2$ in the word embedding.

Alternatively, a distance measure can be derived from knowledge graphs or ontologies. There are numerous resources on options for substitutions (e.g. substitutions to transfer a dish into vegan dish \cite{Steen2010}) that can be used as the basis of specifying distance measures, or distance measures can be defined as combination of word embeddings and ontological knowledge \cite{Shirai2021}.
Furthermore, distance can be calculated so that generalization is penalized. In other words, 
if we have the option of substituting $t$ by $t_1$ or $t_2$, and $t_1$ is relatively similar to $t$ and as specialized in the hierarchy as $t$, whereas $t_2$ is an ancestor of $t$ but $t_2$ is a very general type, then we may set $d$ so that $d(t,t_1)$ is much lower than $d(t,t_2)$, as for instance when $t$ is carrot, $t_1$ is parsnip, and $t_2$ is vegetable.


Assuming we have a distance measure,
we can calculate the cost of a substitution set as in the following definition
which uses summation. But there other simple alternatives such as max (i.e. the distance $d(F(n),t)$ of the substitution $(n,t) \in P\cup S$ that is greater than or equal to the distance $d(F(n'),t')$ of any substitution $(n',t') \in P\cup S)$). 
We will investigate these alternatives in future work.

\begin{definition}
The {\bf cost} of substitution pair $(P,S)$ 
w.r.t distance measure $d$ and typing function $F$ is
$\cost_d(P,S) = \sum_{(n,t) \in P\cup S} d(F(n),t)$.
\end{definition}

\begin{example}
Consider recipe $R$ for a vegetable soup where $\type(R) = F$,
and $\{c1,c2,a1,a2\}\subseteq \nodes(R)$, 
and 
$F(\mbox{\rm c1})$ = $\mbox{\rm raw carrot}$,
$F(\mbox{\rm c2})$ = $\mbox{\rm barley}$,
$F(\mbox{\rm a1})$ = $\mbox{\rm chop carrot}$,
and
$F({\rm a2})$ = $\mbox{\rm soak barley}$.
For a substitution pair $(P,S)$, 
suppose $P$ = $\{ \mbox{\rm (c1, raw onion)}$,  $\mbox{\rm (c2, potato)}  \}$,
and $S$ = $\{ \mbox{\rm (a1, chop onion)}$,   $\mbox{\rm (a2, peel and chop potato)}  \}$.
So 
\[
\begin{array}{ll}
\cost_d(P,S) = & d(\mbox{\rm raw carrot},\mbox{\rm raw onion}) \\
& + d(\mbox{\rm barley},\mbox{\rm potato}) \\
& + d(\mbox{\rm chop carrot}, \mbox{\rm chop onion}) \\
& + d(\mbox{\rm soak barley},\mbox{\rm peel and chop potato})\\
\end{array}
\]
\end{example}

\begin{definition}
A substitution pair $(P,S)$ for recipe $R$ and acceptability pair $X$ is a {\bf preferred substitution pair} w.r.t. distance $d$
iff for all substitution pairs $(P',S')$ for $R$ and $X$, 
$\cost_d(P,S) \leq \cost_d(P',S')$.
\end{definition}

So type substitution allows us to update a typing function, and it can take account of the need for secondary substitutions. Furthermore, it can take into account how drastic the proposed changes are. However, it does not allow changes to the structure of the graph.

\section{Structural substitution in recipes}
\label{section:structuralsubstitution}

We now consider how we can update the structure of the graph by allowing substitution of subgraphs so that a subgraph can be replaced by another subgraph. We start with three subsidiary definitions.

\begin{definition}
For recipes $R$ and $R'$, where $R' \sqsubseteq R$, the {\bf front} set is
$\front(R,R')$ = $(\outputs(R')\setminus\outputs(R)) \cup (\inputs(R')\setminus\inputs(R))$.
\end{definition}

So the front is the set of nodes that are in or out nodes in the subgraph but not in or out nodes in the graph as illustrated in Figures \ref{fig:structural} and \ref{fig:structural:2}. 

\begin{definition}
A recipe $R'$ is an {\bf untrimmed subrecipe } of a recipe $R$,
denoted $R' \sqsubseteq^* R$,
iff $R' \sqsubseteq R$
and for all $\forall a \in \actions(R')$, 
if $(c,a)$ or $(a,c)$ is in $\arcs(R)$, then $c \in \nodes(R')$.
\end{definition}

So an untrimmed subrecipe is such that for each action in the subgraph, all the comestible nodes that are connected in the original graph are in the subgraph.
In Figures \ref{fig:structural} and \ref{fig:structural:2}, both $R_1$ and $R_2$ are untrimmed subrecipes of $R$. 

We also require the following notion of parallel which captures whether two subgraphs connect to the front nodes in the same direction.
For example, in Figure \ref{fig:structural}, both $a2$ and $a9$ connect to the front node $c3$ in the same direction (i.e. $c3$ is the target in both subgraphs).


\begin{definition}
Recipe $R_1$ is {\bf parallel} to recipe $R_2$ w.r.t. recipe $R$
iff
for all $c \in \front(R,R_1)$, 
the following hold:
(1) for all $(c,a) \in \arcs(R_1)$, 
there is a $(c,a') \in \arcs(R_2)$;
and (2) for all $(a,c) \in \arcs(R_1)$, 
there is a $(a',c) \in \arcs(R_2)$.
\end{definition}

We use the above three definitions as conditions in the following definition of structural substitution.

\begin{definition}
Let $R$, $R_1$, and $R_2$, be recipes. 
The {\bf structural substitution} of $R_1$ by $R_2$ in $R$, denoted $R[R_1/R_2]$, is defined as follows: 
If (i) $\front(R,R_1) \subseteq (\inputs(R_2)\cup\outputs(R_2))$, 
(ii) $R_1$ is parallel with $R_2$ w.r.t. $R$,
(iii) $R_1 \sqsubseteq^* R$, 
(iv) $(\nodes(R)\setminus\nodes(R_1)) \cap \nodes(R_2) = \emptyset$,
(v) for all $n \in (\nodes(R)\setminus\nodes(R_1))$, 
and for all $n' \in \nodes(R_2)$,
if $F(n) \simeq F_2(n')$, then $n = n'$,
then $R[R_1/R_2]$ is $(C',A',E',F')$ where 
\begin{enumerate}
\item $C' = (\comestibles(R)\setminus\comestibles(R_1)) \cup \comestibles(R_2)$
\item $A' = (\actions(R)\setminus\actions(R_1)) \cup \actions(R_2)$
\item $E' = (\arcs(R)\setminus\arcs(R_1)) \cup  \arcs(R_2)$
\item for all $n\in C'\cup A'$, 
$F'(n) = F_2(n)$ if $n \in\nodes(R_2)$,
and $F'(n) = F(n)$ if $n \notin\nodes(R_2)$.
\end{enumerate}
otherwise $R[R_1/R_2] = \bot$ (which denotes failure of structural substitution).
\end{definition}

So the nodes, arcs, and labelling, of the graph obtained by substitution are those of $R$ minus those of $R_1$ and plus those of $R_2$ as illustrated in Figures \ref{fig:structural} and \ref{fig:structural:2}. 
The result of structural substitution is a recipe.

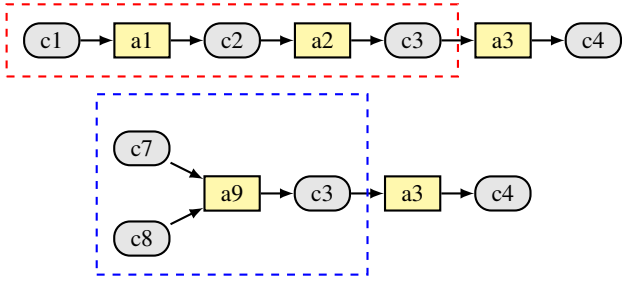
\begin{figure}
\begin{center}
\begin{tikzpicture}
[->,>=latex,thick, scale=0.6,
com/.style={draw,text centered, text width=5mm,
shape=rectangle, rounded corners=6pt,
fill=gray!20,font=\footnotesize},
act/.style={draw,text centered, 
shape=rectangle,text width=5mm,
fill=yellow!40,font=\footnotesize}
]
\node[com] (c1) [] at (0,0) {c1};
\node[act] (a1) [] at (2,0) {a1};
\node[com] (c2) [] at (4,0) {c2};
\node[act] (a2) [] at (6,0) {a2};
\node[com] (c3) [] at (8,0) {c3};
\node[act] (a3) [] at (10,0) {a3};
\node[com] (c4) [] at (12,0) {c4};
\path	(c1) edge node[above] {} (a1);
\path	(a1) edge node[above] {} (c2);
\path	(c2) edge node[above] {} (a2);
\path	(a2) edge node[above] {} (c3);
\path	(c3) edge node[above] {} (a3);
\path	(a3) edge node[above] {} (c4);
\draw[red,thick,dashed] (-1,-0.8) rectangle (9,0.8);
\end{tikzpicture}
\end{center}
\begin{center}
\begin{tikzpicture}
[->,>=latex,thick, scale=0.6,
com/.style={draw,text centered, text width=5mm,
shape=rectangle, rounded corners=6pt,
fill=gray!20,font=\footnotesize},
act/.style={draw,text centered, 
shape=rectangle,text width=5mm,
fill=yellow!40,font=\footnotesize}
]
\node[com] (c7) [] at (4,1) {c7};
\node[com] (c8) [] at (4,-1) {c8};
\node[act] (a9) [] at (6,0) {a9};
\node[com] (c3) [] at (8,0) {c3};
\node[act] (a3) [] at (10,0) {a3};
\node[com] (c4) [] at (12,0) {c4};
\path	(c7) edge node[above] {} (a9);
\path	(c8) edge node[above] {} (a9);
\path	(a9) edge node[above] {} (c3);
\path	(c3) edge node[above] {} (a3);
\path	(a3) edge node[above] {} (c4);
\draw[blue,thick,dashed] (3,-1.8) rectangle (9,2.2);
\end{tikzpicture}
\end{center}
\caption{\label{fig:structural}Example of structural substitution. The top graph refers to recipe $R$, the subgraph in the top dashed box refers to $R_1$, the bottom graph refers to recipe $R'$, and the subrecipe in the bottom dashed box refers to $R_2$. 
Hence, $\front(R,R_1) \subseteq (\inputs(R_2)\cup\outputs(R_2)) = \{c3\}$. 
Suppose $F_1(c3) = F_2(c3)$ where $\type(R_1) = F_1$ and $\type(R_2) = F_2$. 
Hence, $R'$ is the result of the structural substitution  and so $R' = R[R_1/R_2]$. Here, $R_1$ produces $F_1(c3)$ from $F_1(c1)$ and $R_2$  produces $F_2(c3)$ from $F_2(c7)$ and $F_2(c8)$. }
\end{figure}


\begin{proposition}
For any recipes $R$, $R_1$ and $R_2$, 
if $R[R_1/R_2] \neq \bot$, 
then $R[R_1/R_2]$ is a recipe.
\end{proposition}

\begin{proof}
Assume conditions (i) to (v) of the definition for structural substitution hold. 
From (i) and steps 1 to 3 of the definition of structural substitution,
we have for all $n\in \nodes(R)\setminus\nodes(R_1)$,
and for all $n' \in \nodes(R_1)$,
if $(n,n')\in\arcs(R)$,
then $(n,n')\in\arcs(R[R_1/R_2])$,
and 
if $(n',n)\in\arcs(R)$,
then $(n',n)\in\arcs(R[R_1/R_2])$.
So $R$ connects with $R_1$ using the same arcs as $R[R_1/R_2]$ connects with $R_2$.
From (ii),
for all $c \in \front(R,R_1)$,
if $R_1$ connects with $c$ with an incoming (respectively outgoing) arc,
then $R_2$ connects with $c$ with an incoming (respectively outgoing) arc.
So $R_2$ connects with the nodes in $\front(R,R_1)$
with arcs in the same direction as $R_1$,
and furthermore from (iii), $R_2$ connects to all nodes in $\front(R,R_1)$.
From (iv), the nodes in $R_2$ do not introduce any cycles with $\nodes(R)\setminus\nodes(R_1)$.
From (v), and step 4 of the definition of structural substitution,
$F'$ is typing function for recipe according to Definition \ref{def:recipe}.
Therefore, from these assumptions, $R[R_1/R_2]$ is a recipe.
\end{proof}


\begin{figure}
\begin{center}
\begin{tikzpicture}
[->,>=latex,thick, scale=0.6,
com/.style={draw,text centered, text width=5mm,
shape=rectangle, rounded corners=6pt,
fill=gray!20,font=\footnotesize},
act/.style={draw,text centered, 
shape=rectangle,text width=5mm,
fill=yellow!40,font=\footnotesize}
]
\node[com] (c1) [] at (0,0) {c1};
\node[act] (a1) [] at (2,0) {a1};
\node[com] (c2) [] at (4,0) {c2};
\node[act] (a2) [] at (6,0) {a2};
\node[com] (c3) [] at (8,0) {c3};
\node[act] (a3) [] at (10,0) {a3};
\node[com] (c4) [] at (12,0) {c4};
\path	(c1) edge node[above] {} (a1);
\path	(a1) edge node[above] {} (c2);
\path	(c2) edge node[above] {} (a2);
\path	(a2) edge node[above] {} (c3);
\path	(c3) edge node[above] {} (a3);
\path	(a3) edge node[above] {} (c4);
\draw[red,thick,dashed] (3,-0.8) rectangle (9,0.8);
\end{tikzpicture}
\end{center}
\begin{center}
\begin{tikzpicture}
[->,>=latex,thick, scale=0.6,
com/.style={draw,text centered, text width=5mm,
shape=rectangle, rounded corners=6pt,
fill=gray!20,font=\footnotesize},
act/.style={draw,text centered, 
shape=rectangle,text width=5mm,
fill=yellow!40,font=\footnotesize}
]
\node[com] (c1) [] at (0,0) {c1};
\node[act] (a1) [] at (2,0) {a1};
\node[com] (c2) [] at (4,0) {c2};
\node[act] (a4) [] at (4,-1.5) {a4};
\node[com] (c5) [] at (6,-1.5) {c5};
\node[act] (a6) [] at (8,-1.5) {a6};
\node[com] (c3) [] at (8,0) {c3};
\node[act] (a3) [] at (10,0) {a3};
\node[com] (c4) [] at (12,0) {c4};
\path	(c1) edge node[above] {} (a1);
\path	(a1) edge node[above] {} (c2);
\path	(c2) edge node[above] {} (a4);
\path	(a4) edge node[above] {} (c5);
\path	(c5) edge node[above] {} (a6);
\path	(a6) edge node[above] {} (c3);
\path	(c3) edge node[above] {} (a3);
\path	(a3) edge node[above] {} (c4);
\draw[blue,thick,dashed] (3,-2.3) rectangle (9,0.8);
\end{tikzpicture}
\end{center}
\caption{\label{fig:structural:2}Example of structural substitution. The top graph refers to recipe $R$, the subgraph in the top dashed box refers $R_1$, the bottom graph refers to recipe $R'$, and the subrecipe in the bottom dashed box refers $R_2$. 
Hence, $\front(R,R_1) \subseteq (\inputs(R_2)\cup\outputs(R_2)) = \{c2,c3\}$. 
Suppose $F_1(c2) = F_2(c2)$ and $F_1(c3) = F_2(c3)$ where $\type(R_1) = F_1$ and $\type(R_2) = F_2$.
Hence, $R'$ is the result of the structural substitution  and so $R' = R[R_1/R_2]$. Here, $R_1$ produces $F_1(c3)$ from $F_1(c2)$ using action $F_1(a2)$ and $R_2$  produces $F_2(c3)$ from $F_2(c2)$ using actions $F_2(a4)$ and $F_2(a6)$. }
\end{figure}
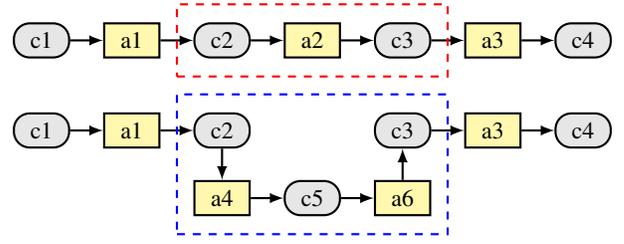

%
%
%

We can always turn one recipe into another recipe using structural substitution.

\begin{proposition}
For any recipes $R$ and $R'$,
there are recipes $R_1$ and $R_2$,
such that $R[R_1/R_2] = R'$.
\end{proposition}

\begin{proof}
The simplest option to demonstrate this is to assume $R = R_1$, and $R' = R_2$.
From $R = R_1$, $\front(R,R_1) = \emptyset$ holds,
and so (i) and (ii) are satisfied.
Also from $R = R_1$, $R_1 \subseteq^{*} R$ holds,
and so (iii) is satisfied. 
In addition, from $R = R_1$, $\nodes(R)\setminus\nodes(R_1) = \emptyset$ holds,
and so (iv) and (v) are satisfied. 
So $R[R_1/R_1]$ is $(C',A',E',F')$
where
$C' = \comestibles(R_2)$,
$A' = \actions(R_2)$,
$E' = \arcs(R_2)$,
and 
for all $n\in C'\cup A'$, $n \in\nodes(R_2)$
and so  $F'(n) = F_2(n)$
when $\type(R_2)$ is $F_2$. 
Therefore, 
$R[R_1/R_2] = R'$.
\end{proof}





Structural substitution satisfies some simple properties 
including reflexivity (i.e. $R[R_1/R_1] = R$), 
and reversibility (i.e. $(R[R_1/R_2])[R_2/R_1]$), 
but it does not satisfy transitivity, or associativity. 
The following result shows that structural substitution subsumes type substitution.

\begin{proposition}
For any recipe $R$, and substitution pair $(P,S)$,
there is a structural substitution $[R_1/R_2]$
such that $R[R_1/R_2] = R\otimes(P,S)$. 
\end{proposition}

\begin{proof}
If $R\otimes(P,S) = \bot$, then it is straightforward to specify recipes $R_1$ and $R_2$ so that $R[R_1/R_2] = \bot$.
Now suppose $R\otimes(P,S) \neq \bot$. 
Define $R_1$ and $R_2$ so that $\graph(R_1) = \graph(R_2)$.
Hence $\graph(R) = \graph(R\otimes(P,S))$. 
Let $\type(R) = F$ and $\type(R_2) = F_2$. 
Define $F_2$ so that if $(n,t) \in P\cup S$, $F_2(n) = t$, 
otherwise $F_2(n) = F(n)$.
\end{proof}

As with type substitutions, we can have structural substitutions that we regard as primary; they are subgraphs that we want to replace. But these can also have knock-on effects that call for  secondary substitutions. 
We define this below and illustrate in Figure \ref{fig:spaghetti:3}.

\begin{definition}
A sequence of substitutions $[R_{j+1}/R'_{j+1}],\ldots,[R_k/R'_k]$ 
is a {\bf secondary substitution sequence} for recipe $R$, 
primary substitution sequence $[R_1/R'_1],\ldots,[R_j/R'_j]$, and acceptability tuples $X$ 
iff $R[R_1/R'_1],\ldots,[R_k/R'_k]$ is acceptable w.r.t. $X$. 
\end{definition}


\begin{figure}
\begin{center}
\begin{tikzpicture}
[->,>=latex,thick, scale=0.75,
com/.style={draw,text centered, text width=40mm,
shape=rectangle, rounded corners=6pt,
fill=gray!20,font=\footnotesize},
act/.style={draw,text centered, 
shape=rectangle,text width=40mm,
fill=yellow!40,font=\footnotesize}
]
\node[com] (c0) [text width=35mm] at (0,8) {c0: boiling salted water};
\node[com] (c1) [text width=25mm] at (5,8) {c1: spaghetti};
\node[act] (a1) [text width=35mm] at (0,6.5) {a1: boil pasta for 10 min};
\node[com] (c2) [text width=35mm] at (0,5) {c2: cooked spaghetti};
\node[act] (a2) [text width=30mm] at (6,5) {a2: mix  and  heat};
\node[com] (c5) [text width=35mm] at (6,3.5) {c5: heated  bolognese sauce};
\node[act] (a3) [text width=25mm] at (0,3.5) {a3: drain and pour  in bowl};
\node[com] (c6b) [text width=20mm] at (-1,2) {c6: pasta water};
\node[com] (c6) [text width=30mm] at (3,2) {c7: spaghetti in bowl};
\node[act] (a4) [text width=50mm] at (5,0.5) {a4: pour bolognese sauce on spaghetti};
\node[com] (c7) [text width=40mm] at (5,-1) {c8: spaghetti bolognese};
\node[com] (c9) [text width=36mm] at (6,6.5) {c9: jar of bolognese sauce};
\path	(c0) edge node[above] {} (a1);
\path	(c1) edge node[above] {} (a1);
\path	(a1) edge node[above] {} (c2);
\path	(c9) edge node[above] {} (a2);
\path	(a2) edge node[above] {} (c5);
\path	(c2) edge node[above] {} (a3);
\path	(a3) edge node[above] {} (c6);
\path	(a3) edge node[above] {} (c6b);
\path	(c6) edge node[above] {} (a4);
\path	(c5) edge node[above] {} (a4);
\path	(a4) edge node[above] {} (c7);
\draw[blue,thick,dashed] (3,2.7) rectangle (8.8,7.1);
\draw[red,thick,dashed,-] (0.7,-1.7) -- (0.7,2.6) -- (2.7,2.6) -- (2.7,4.4) -- (8.7,4.4) -- (8.7,-1.7) -- (0.6,-1.7);
\end{tikzpicture}
\end{center}
\caption{\label{fig:spaghetti:3}Let $R$ be the recipe given in Figure \ref{fig:spaghetti:1} and let $R_1\sqsubseteq R$ (respectively $R'_1\sqsubseteq R$) be the subgraph given by the nodes $\{c3,c4,c5,a2\}$ (respectively $\{c7,c8,a4\}$) in Figure \ref{fig:spaghetti:1}. Let $R_2$ be the blue box (the lower dashed box above), and let $R'_2$ be the red box (the upper dashed box above). Suppose $[R_1/R_2]$ is a primary substitution, and $[R'_1/R'_2]$ is a secondary substitution, then the recipe above is $(R[R_1/R_2])[R'_1/R'_2]$. }
\end{figure}
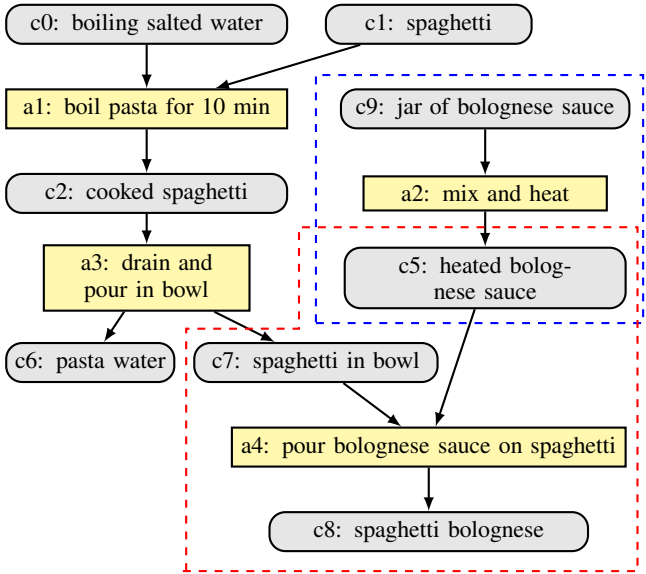

We can calculate the cost of any structural substitution by using a distance measure. For a structural substitution $[R_1/R_2]$, this could be a function of the graph edit distance between the $R_1$ and $R_2$, and the distance between the ingredients and actions in $R_1$ and those in $R_2$. We will investigate specific definitions in future work.

\section{Comparison with the literature}
\label{section:literature}

We consider how our proposal relates to the existing literature in the following areas.

{\bf Graphical representation of recipes.}
There are several proposals for graphical representation of recipes
with motivations including: as
a basis for searching and/or adapting recipes \cite{Shao2009,Dufour2012,Bergmann2013,Muller2015,Chang2018,Chu2021,Galanis2022}; as a target language for the ouput of NLP of free text recipes \cite{Wang2008,Schumacher2012,Pan2020,Yamakata2020}. 
However, little consideration has been given to how a graphical model can be used as a knowledge representation and reasoning formalism. 
A representation of recipes as plans was recently proposed in \cite{Pallagani2022}. This is an interesting alternative, but it does not support reasoning over recipes.

{\bf Ontologies and knowledge graphs.}
The use of ontologies and knowledge graphs for formally representing food or recipes has recently been gaining interest. A prominent example is the FoodOn ontology \cite{dooley2018,Dooley2022}. This has the form of a multi-faceted taxonomy, 
which covers a very broad range of food products and food processing steps. It cannot, however, support rich forms of reasoning due to its simple structure and small number of relationships. FoodKG is a large-scale food knowledge graph, which integrates nutrition information, general food substitutions, recipe data and food taxonomies \cite{Haussmann2019}. It does not support formal reasoning over recipes, but is a valuable resource that we plan to exploit in future implementations of our model.
An ontology design pattern for ingredient substitution in recipes was recently proposed in \cite{Lawrynowicz2022}. 
It models recipes as a set of ingredients and a set of instructions and ingredient substitution as transformations of these sets.
Its value, however, is mostly representational, as it  does not support reasoning over recipes and ingredient substitution.


{\bf Substitution}
Candidates for substitution can be idenitified by analysing recipes
\cite{Shidochi2009,Teng2012}. 
They can also be identified by combining explicit information about the ingredients in FoodKG, and implicit information from word embeddings \cite{Shirai2021}.
None of these methods provide a formalism for representing or reasoning with recipes, 
but they could be used for finding candidates for substitution for use in our framework.  
Finally, there is a proposal for using a formalism to capture features of objects (namely, shape, material, and role of the object) and then reason with that knowledge to identify alternative uses \cite{Olteteanu2016}. 
Potentially, this logic-based approach could be adapted for recipes by perhaps drawing on the approach in our proposal. 

\section{Discussion}
\label{section:discussion}

Cooking involves reasoning about comestibles and actions on them. 
It also involves the commonsense ability to repurpose recipes, in particular through substitution. 
In this paper, we have provided a formalism based on labelled bipartite graphs for representing and reasoning with recipes.
In future work, we plan to further develop the approach to allow for consideration of equipment and locations, timing of actions, and quantities of comestibles. We also plan to further develop the usage of distance in determining better substitutions. We will also investigate two possible implementations of the graphical formalism. The first one will use RDF and SHACL for the representation and validation of recipes, and will enable the integration of information from  resources such as FoodOn and FoodKG. The second will rely on a translation of the graphical formalism into a logical one (e.g. ASP) and will thereby support automated reasoning with recipes including execution of recipes, and automated methods for substitution.




\section*{Acknowledgements}

This research is funded by the Leverhulme Trust via the {\em Repurposing of
Resources: from Everyday Problem Solving through to Crisis Management} project (2022-2025).

\bibliographystyle{kr}
\bibliography{subgraph}

\begin{thebibliography}{}

\bibitem[\protect\citeauthoryear{Bergmann, M{\"{u}}ller, and
  Wittkowsky}{2013}]{Bergmann2013}
Bergmann, R.; M{\"{u}}ller, G.; and Wittkowsky, D.
\newblock 2013.
\newblock Workflow clustering using semantic similarity measures.
\newblock In {\em Proceedings of {KI}'13}, volume 8077 of {\em Lecture Notes in
  Computer Science},  13--24.
\newblock Springer.

\bibitem[\protect\citeauthoryear{Bikakis \bgroup et al\mbox.\egroup
  }{2021}]{Bikakis2021}
Bikakis, A.; Dickens, L.; Hunter, A.; and Miller, R.
\newblock 2021.
\newblock Repurposing of resources: from everyday problem solving through to
  crisis management.
\newblock {\em CoRR} abs/2109.08425.

\bibitem[\protect\citeauthoryear{Chang \bgroup et al\mbox.\egroup
  }{2018}]{Chang2018}
Chang, M.; Guillain, L.~V.; Jung, H.; Hare, V.~M.; Kim, J.; and Agrawala, M.
\newblock 2018.
\newblock Recipescape: An interactive tool for analyzing cooking instructions
  at scale.
\newblock In {\em Proceedings of {CHI}'18},  1–12.
\newblock Association for Computing Machinery.

\bibitem[\protect\citeauthoryear{Chu}{2021}]{Chu2021}
Chu, J.
\newblock 2021.
\newblock Recipe bot: The application of conversational ai in home cooking
  assistant.
\newblock In {\em Proceedings of {ICBASE}'21},  696--700.
\newblock IEEE.

\bibitem[\protect\citeauthoryear{Dooley \bgroup et al\mbox.\egroup
  }{2018}]{dooley2018}
Dooley, D.~M.; Griffiths, E.~J.; Gosal, G. P.~S.; Buttigieg, P.~L.; Hoehndorf,
  R.; Lange, M.; Schriml, L.~M.; Brinkman, F. S.~L.; and Hsiao, W. W.~L.
\newblock 2018.
\newblock Foodon: a harmonized food ontology to increase global food
  traceability, quality control and data integration.
\newblock {\em NPJ Science of Food} 2.

\bibitem[\protect\citeauthoryear{Dooley \bgroup et al\mbox.\egroup
  }{2022}]{Dooley2022}
Dooley, D.; Weber, M.; Ibanescu, L.; Lange, M.; Chan, L.; Soldatova, L.; Yang,
  C.; Warren, R.; Shimizu, C.; McGinty, H.~K.; and Hsiao, W.
\newblock 2022.
\newblock Food process ontology requirements.
\newblock {\em Semantic Web}  1--32.

\bibitem[\protect\citeauthoryear{Dufour{-}Lussier \bgroup et al\mbox.\egroup
  }{2012}]{Dufour2012}
Dufour{-}Lussier, V.; Ber, F.~L.; Lieber, J.; Meilender, T.; and Nauer, E.
\newblock 2012.
\newblock Semi-automatic annotation process for procedural texts: An
  application on cooking recipes.
\newblock {\em CoRR} abs/1209.5663.

\bibitem[\protect\citeauthoryear{Galanis and Papakostas}{2022}]{Galanis2022}
Galanis, N.-I., and Papakostas, G.~A.
\newblock 2022.
\newblock An update on cooking recipe generation with machine learning and
  natural language processing.
\newblock In {\em Proceedings of IEEE AIC'22},  739--744.

\bibitem[\protect\citeauthoryear{Haussmann \bgroup et al\mbox.\egroup
  }{2019}]{Haussmann2019}
Haussmann, S.; Seneviratne, O.; Chen, Y.; Ne’eman, Y.; Codella, J.; Chen,
  C.-H.; McGuinness, D.~L.; and Zaki, M.~J.
\newblock 2019.
\newblock {FoodKG: A Semantics-Driven Knowledge Graph for Food Recommendation}.
\newblock In {\em Proceedings of ISWC'19},  146–162.
\newblock Springer-Verlag.

\bibitem[\protect\citeauthoryear{{\L}awrynowicz \bgroup et al\mbox.\egroup
  }{2022}]{Lawrynowicz2022}
{\L}awrynowicz, A.; Wr\'{0}blewska, A.; Adrian, W.~T.; Kulczyński, B.; and
  Gramza-Michałowska, A.
\newblock 2022.
\newblock Food recipe ingredient substitution ontology design pattern.
\newblock {\em Sensors} 22(3):1095.

\bibitem[\protect\citeauthoryear{Mikolov \bgroup et al\mbox.\egroup
  }{2013}]{Mikolov2013}
Mikolov, T.; Chen, K.; Corrado, G.; and Dean, J.
\newblock 2013.
\newblock Efficient estimation of word representations in vector space.
\newblock {\em CoRR} 1301.3781.

\bibitem[\protect\citeauthoryear{M{\"{u}}ller and Bergmann}{2015}]{Muller2015}
M{\"{u}}ller, G., and Bergmann, R.
\newblock 2015.
\newblock Cookingcake: {A} framework for the adaptation of cooking recipes
  represented as workflows.
\newblock In {\em Workshop Proceedings of {ICCBR}'15)}, volume 1520 of {\em
  {CEUR} Workshop Proceedings},  221--232.
\newblock CEUR-WS.org.

\bibitem[\protect\citeauthoryear{Olteteanu and Falomir}{2016}]{Olteteanu2016}
Olteteanu, A., and Falomir, Z.
\newblock 2016.
\newblock {Object replacement and object composition in a creative cognitive
  system. Towards a computational solver of the Alternative Uses Test}.
\newblock {\em {Cognitive Systems Research}} 39:15--32.

\bibitem[\protect\citeauthoryear{Pallagani \bgroup et al\mbox.\egroup
  }{2022}]{Pallagani2022}
Pallagani, V.; Ramamurthy, P.; Khandelwal, V.; Venkataramanan, R.; Lakkaraju,
  K.; Aakur, S.~N.; and Srivastava, B.
\newblock 2022.
\newblock {A Rich Recipe Representation as Plan to Support Expressive Multi
  Modal Queries on Recipe Content and Preparation Process}.
\newblock {\em {CoRR}} abs/2203.17109.

\bibitem[\protect\citeauthoryear{Pan \bgroup et al\mbox.\egroup
  }{2020}]{Pan2020}
Pan, L.-M.; Chen, J.; Wu, J.; Liu, S.; Ngo, C.-W.; Kan, M.-Y.; Jiang, Y.; and
  Chua, T.-S.
\newblock 2020.
\newblock Multi-modal cooking workflow construction for food recipes.
\newblock In {\em Proceedings of ACM Multimedia'20},  1132–1141.

\bibitem[\protect\citeauthoryear{Pellegrini \bgroup et al\mbox.\egroup
  }{2021}]{Pellegrini2021}
Pellegrini, C.; {\"{O}}zsoy, E.; Wintergerst, M.; and Groh, G.
\newblock 2021.
\newblock Exploiting food embeddings for ingredient substitution.
\newblock In {\em Proceedings of the 14th International Joint Conference on
  Biomedical Engineering Systems and Technologies, {BIOSTEC} 2021, Volume 5:
  HEALTHINF},  67--77.
\newblock {SCITEPRESS}.

\bibitem[\protect\citeauthoryear{Pennington, Socher, and
  Manning}{2014}]{Pennington2014}
Pennington, J.; Socher, R.; and Manning, C.
\newblock 2014.
\newblock Glove: Global vectors for word representation.
\newblock In {\em Proceedings of {EMNLP}'14},  1532--1543.
\newblock Association for Computational Linguistics.

\bibitem[\protect\citeauthoryear{Schumacher \bgroup et al\mbox.\egroup
  }{2012}]{Schumacher2012}
Schumacher, P.; Minor, M.; Walter, K.; and Bergmann, R.
\newblock 2012.
\newblock Extraction of procedural knowledge from the web: a comparison of two
  workflow extraction approaches.
\newblock In {\em Proceedings of {WWW}'12},  739--747.
\newblock {ACM}.

\bibitem[\protect\citeauthoryear{Shao, Sun, and Chen}{2009}]{Shao2009}
Shao, Q.; Sun, P.; and Chen, Y.
\newblock 2009.
\newblock {WISE:} {A} workflow information search engine.
\newblock In {\em Proceedings of {ICDE}'09},  1491--1494.
\newblock {IEEE} Computer Society.

\bibitem[\protect\citeauthoryear{Shidochi \bgroup et al\mbox.\egroup
  }{2009}]{Shidochi2009}
Shidochi, Y.; Takahashi, T.; Ide, I.; and Murase, H.
\newblock 2009.
\newblock Finding replaceable materials in cooking recipe texts considering
  characteristic cooking actions.
\newblock In {\em Proceedings of the ACM Multimedia'09 Workshop on Multimedia
  for Cooking and Eating Activities},  9–14.
\newblock Association for Computing Machinery.

\bibitem[\protect\citeauthoryear{Shirai \bgroup et al\mbox.\egroup
  }{2021}]{Shirai2021}
Shirai, S.~S.; Seneviratne, O.; Gordon, M.~E.; Chen, C.-H.; and McGuinness,
  D.~L.
\newblock 2021.
\newblock Identifying ingredient substitutions using a knowledge graph of food.
\newblock {\em Frontiers in Artificial Intelligence} 3:621766.

\bibitem[\protect\citeauthoryear{Steen and Newman}{2010}]{Steen2010}
Steen, C., and Newman, J.
\newblock 2010.
\newblock {\em The Complete Guide to Vegan Food Substitutions}.
\newblock Fair Winds Press.

\bibitem[\protect\citeauthoryear{Teng, Lin, and Adamic}{2012}]{Teng2012}
Teng, C.-Y.; Lin, Y.-R.; and Adamic, L.~A.
\newblock 2012.
\newblock Recipe recommendation using ingredient networks.
\newblock In {\em Proceedings of ACM Web Science Conference}, WebSci '12,
  298–307.
\newblock Association for Computing Machinery.

\bibitem[\protect\citeauthoryear{Wang \bgroup et al\mbox.\egroup
  }{2008}]{Wang2008}
Wang, L.; Li, Q.; Li, N.; Dong, G.; and Yang, Y.
\newblock 2008.
\newblock Substructure similarity measurement in chinese recipes.
\newblock In {\em Proceedings of WWW'08},  979–988.
\newblock Association for Computing Machinery.

\bibitem[\protect\citeauthoryear{Yamakata, Mori, and
  Carroll}{2020}]{Yamakata2020}
Yamakata, Y.; Mori, S.; and Carroll, J.
\newblock 2020.
\newblock {E}nglish recipe flow graph corpus.
\newblock In {\em Proceedings of {LREC}'20},  5187--5194.
\newblock European Language Resources Association.

\end{thebibliography}

\end{document}